\documentclass[11pt,leqno]{amsart}

\usepackage{amsthm}
\usepackage{url}
\usepackage{hyperref}

\usepackage{graphicx}
\usepackage{subfigure} 
\usepackage{amsmath,esint}
\usepackage{booktabs}
\usepackage{amssymb}
\usepackage{enumerate}
\usepackage{color}
\usepackage{times}

\newtheorem{counter}{Counter}
\newtheorem{theorem}[counter]{Theorem}
\newtheorem{definition}[counter]{Definition}
\newtheorem{lemma}[counter]{Lemma}
\newtheorem{proposition}[counter]{Proposition}

\theoremstyle{remark}
\newtheorem{remark}[counter]{Remark}
\providecommand{\keywords}[1]{\textbf{\emph{Keywords: }} #1}

\newenvironment{listi}
  {\begin{list}
 {(\roman{broj1}) }
 {\usecounter{broj1}
  \setlength{\itemindent}{-1pt}
  \setlength{\listparindent}{-1pt}
  \setlength{\itemsep}{-2pt}}
  \setlength{\labelwidth}{30pt}
  \setlength{\parsep}{0pt}
  }
  {\end{list}}

\newcommand{\x}{\mathbf{x}}
\newcommand{\p}{\mathbf{p}}

\newcommand{\C}{ \mathcal{C}}

\newcommand{\X}{X}
\newcommand{\Y}{Y}
\newcommand{\G}{ \mathcal{G}}
\newcommand{\converges}[1]{ \overset{#1}{\longrightarrow}}

\newcommand{\1}{\mathbf{1}}
\newcommand{\cut}{\mathrm{Cut}}
\newcommand{\bal}{\mathrm{Bal}}

\newcommand{\rd}{\mathrm{d}}

\newcommand{\R}{\mathbb{R}}

\newcommand{\nc}{\normalcolor}
\newcommand{\real}{\mathbb{R}}

\newcommand{\N}{\mathbb{N}}

\newcommand{\te}{\textrm}

\newcommand{\veps}{\varepsilon}

\newcommand{\tTV}{\widetilde{TV}}

\newcommand{\ind}{\mathrm{Ind}}
\newcommand{\one}{\1}
\newcommand{\drho}{\rho(x) \; \rd x}
\newcommand{\D}{D}

\newcommand{\Rd}{\mathbb{R}^{d}}

\newcommand{\ra}[1]{\renewcommand{\arraystretch}{#1} }

\DeclareMathOperator{\id}{Id}

\DeclareMathOperator{\mean}{mean}

\DeclareMathOperator{\supp}{supp}

\newcommand{\Part}{\mathcal{M}}

\numberwithin{equation}{section}

\newcounter{broj1}

\title{Consistency of Cheeger and Ratio Graph Cuts}

\author[ N. Garcia Trillos, D. Slep\v{c}ev, J. von Brecht, T. Laurent and X. Bresson]{Nicol\'as Garc\'ia Trillos$^1$, Dejan Slep\v{c}ev$^1$,  James von Brecht$^2$, Thomas Laurent$^3$, Xavier Bresson$^4$.}

\address{$^1$
Department of Mathematical Sciences, Carnegie Mellon University, Pittsburgh, PA, 15213, USA. \\
tel. +412 268-2545, 
emails: ngarciat@andrew.cmu.edu, slepcev@math.cmu.edu }

\address{$^2$ Department of Mathematics and Statistics, California State University, Long Beach
Long Beach, CA 90840, USA. \\
email: James.vonBrecht@csulb.edu
}

\address{$^3$ Department of Mathematics, Loyola Marymount University, 1 LMU Dr, Los Angeles, CA 90045, Usa.\\ email: Thomas.Laurent@lmu.edu}

\address{$^4$
Institute of Electrical Engineering,  Swiss Federal Institute of Technology (EPFL), 1015 Lausanne, Switzerland \\ email: xavier.bresson@epfl.ch  }

\begin{document}
\keywords{data clustering, balanced cut, consistency, Gamma convergence, graph partitioning}

\subjclass{62H30, 62G20, 49J55, 91C20, 68R10, 60D05}


\date{\today}
\maketitle

\begin{abstract}
This paper establishes the consistency of  a family of graph-cut-based algorithms for clustering of data clouds. We consider point clouds obtained as samples of a ground-truth measure.  We investigate approaches to clustering based on minimizing objective functionals defined on proximity graphs of the given sample. Our focus is on functionals based on graph cuts like the Cheeger and ratio cuts. We show that minimizers of the these cuts converge as the sample size increases to a minimizer of a corresponding continuum  cut (which partitions the ground truth measure). Moreover, we obtain sharp conditions on how the connectivity radius can be scaled with respect to the number of sample points for the consistency to hold. We provide results for two-way and  for multiway  cuts. Furthermore we provide numerical experiments that illustrate the results and explore the optimality of scaling in dimension two.
%
\end{abstract}

\section{Introduction}\label{sec:intro}
Partitioning data clouds in meaningful clusters is one of the fundamental tasks in data analysis and machine learning. A large class of the approaches, relevant to high-dimensional data, relies on  creating a graph out of the data cloud by connecting nearby points. This allows one to leverage the geometry of the data set and obtain high quality clustering.  Many of the graph-clustering  approaches  are based on
 optimizing an objective function which measures the quality of the partition. The basic desire to obtain clusters which are well separated leads to the introduction of objective functionals which penalize the size of cuts between clusters. The desire to have clusters of meaningful size and for approaches to be robust to outliers leads to the introduction of "balance" terms and objective functionals such as  
 Cheeger cut (closely related to edge expansion) \cite{ARV, BLasymetric, BLUV12, KanVemVet04, pro:SzlamBresson10}, ratio cut \cite{HagKah, pro:HeinSetzer11TightCheeger, tutorial:vonluxburg, WeiChe89}, normalized cut \cite{ACPP2012, art:ShiMalik00NCut, tutorial:vonluxburg}, and conductance (sparsest cut) \cite{ARV, KanVemVet04, Nibble2004}. Such functionals can be extended to 
 treat multiclass partitioning \cite{BLUV13, YuShi03}.
 The balanced cuts above have been widely studied theoretically and used computationally. The algorithms of \cite{NibblePageRank, Nibble2004, Nibble2013} utilize local clustering algorithms to compute balanced cuts of large graphs. Total variation based algorithms \cite{BLUV12, BLUV13,pro:HeinBuhler10OneSpec,pro:HeinSetzer11TightCheeger,  pro:SzlamBresson10} are also used to optimize either the conductance or the edge expansion of a graph.
 Closely related are the spectral approaches to clustering  \cite{art:ShiMalik00NCut, tutorial:vonluxburg}
 which can be seen as a relaxation of the normalized cuts.
\medskip

In this paper we consider data clouds, $\X_n = \{ \x_1, \dots, \x_n\}$, which have been obtained as i.i.d. 
samples of a measure  $\nu$ with density $\rho$ on a bounded domain $D$. The measure $\nu$ represents the ground truth that $\X_n$ is a sample of. 
 In the large sample limit, $n \to \infty$,  clustering methods should exhibit \emph{consistency}. That is, the clustering of the data sample $\X_n$ should converge as $n \to \infty$ toward a specific clustering of the underlying sample domain.   In this paper we characterize in a precise manner  when and how the minimizers of a ratio and Cheeger graph cuts  converge towards a suitable partition of the domain.  We define the discrete and continuum objective functionals considered in Subsections 
\ref{dp1} and \ref{cp1} respectively, and informally state our result in Subsection \ref{pdc}. 

An important consideration when investigating consistency of algorithms is how the graphs on $\X_n$ are constructed.
In simple terms, when building a graph on $\X_n$ one sets a length scale $\veps_n$ such that edges between vertices in $\X_n$ are given significant weights if the distance of points they connect is $\veps_n$ or less. In some way this sets the length scale over which the geometric  information is averaged when setting up the graph. Taking smaller $\veps_n$ is desirable because it is computationally less expensive and gives better resolution, but there is a price. Taking $\veps_n$ small increases the error due to randomness and in fact if $\veps_n$ is too small
the resulting graph may not represent the geometry of $D$ well and consequently the discrete graph cut may be very far from the desired one.
In our work we determine precisely how small $\veps_n$ can be taken for the consistency to hold. 
We obtain consistency results  both for two-way and multiway  cuts.

 To prove our results we use the variational notion of convergence known as the $\Gamma$-convergence.
It is one of the standard tools of modern applied analysis to consider a limit of a family of variational problems  \cite{Braides, Dalmaso}.  In the recent work \cite{GTS}, this notion was developed in the random discrete setting designed for the study of consistency of minimization problems on random point clouds. 
In particular the proof of $\Gamma$-convergence of total variation on graphs proved there provides the technical backbone of this paper. The approach we take is general and flexible and we believe suitable for the study of many problems involving large sample limits of minimization problems on graphs. 
\medskip

\textbf{Background on consistency of clustering algorithms and related problems.} 
Consistency of clustering algorithms has been considered for a number of approaches.
Pollard \cite{Pollard1981} has proved the consistency of $k$-means clustering. 
Consistency for a class of single linkage clustering algorithms was shown by Hartigan \cite{Hartigan1981}.
Arias-Castro and Pelletier have proved the consistency of maximum variance unfolding \cite{ariasMVU}.
 Consistency of spectral clustering was rigorously considered by von Luxburg, Belkin, and Bousquet \cite{LBB2004,LBB2008}. These works show the convergence of all eigenfunctions of the graph laplacian for fixed length scale $\veps_n = \veps$ which results in the limiting (as $n \to \infty$) continuum problem beeing a nonlocal one. Belkin and Niyogi  \cite{BN2006} consider the spectral problem (Laplacian eigenmaps)  and show that there exists a sequence $\veps_n \to 0$ such that in the limit the (manifold) Laplacian is recovered, however no rate at which $\veps_n$ can go to zero is provided. 
Consistency of normalized cuts was considered by Arias-Castro, Peletier, and Pudlo \cite{ACPP2012}
who provide a rate on $\veps_n \to 0$ under which the minimizers of the discrete cut functionals minimized over a specific family of subsets of $\X_n$ converge to the continuum Cheeger set. 
Our work improves on \cite{ACPP2012} in several ways. We minimize the discrete functionals over all discrete partitions on $\X_n$ as it is considered in practice and prove the result for the optimal, in terms of scaling, range of rates at which $\veps_n \to 0$ as $n \to \infty$. 

There are also a number of works which investigate how well the discrete functionals approximate the continuum ones for a particular function. Among them are
works by Belkin and Niyogi \cite{bel_niy_LB},   
Gin\'e and Koltchinskii \cite{GK},
 Hein, Audibert,  von Luxburg \cite{hein_audi_vlux05}, 
 Singer \cite{Singer} and Ting, Huang, and  Jordan \cite{THJ}.
 Maier, von Luxburg and Hein \cite{MvLH12} considered pointwise convergence  
  for Cheeger and normalized cuts, both for the
  geometric and kNN graphs and obtained a range of scalings of graph construction on $n$
  for the convergence to hold.
 While these results are quite valuable, we point out that they do not imply that the minimizers of discrete objective functionals are close to minimizers of continuum functionals.

 \subsection{Graph partitioning} \label{dp1}
The balanced cut objective functionals we consider are relevant to general graphs (not just ones obtained from point clouds). We introduce them here.
 
Given a weighted graph $\G = (\X,W)$  with the vertex set $\X=\{ \x_1, \ldots, \x_n\}$ and the  weight matrix $W = \{w_{ij}\}_{1\le i, j \le n}$, the balanced graph cut problems we consider take the form
\begin{equation}  \label{P1}
\text{Minimize \;\;\;  } \frac{ \cut(\Y,\Y^c) } { \bal(\Y,\Y^c) } := \frac{\sum_{\x_i \in \Y} \sum_{\x_j \in \Y^c} w_{ij} } { \bal(\Y,\Y^c) } \quad  \text{over all nonempty $\Y \subsetneq \X$.}
\end{equation}
That is, we consider the class of problems with $\cut(\Y,\Y^c)$ as the numerator together with different balance terms. For $Y \subset X$ let $|Y|$ be the ratio between the number of vertices in $Y$ and the number of vertices in $X$. 
Well-known balance terms include
\begin{equation}  \label{B}
\bal_{ {\rm R}}(\Y,\Y^c) = 2|\Y| |\Y^c|   \qquad \text{and} \qquad \bal_{{\rm C}}(\Y,\Y^c)= \min(|\Y|, |\Y^c|), 
\end{equation}
which correspond to  Ratio Cut  \cite{HagKah, pro:HeinSetzer11TightCheeger, tutorial:vonluxburg, WeiChe89} and Cheeger Cut   \cite{ARV,  art:Cheeger70RatioCut, book:Chung97Spectral, KanVemVet04} respectively \footnote{The factor of 2 in the definition of $\bal_{ {\rm R}}(\Y,\Y^c)$ is introduced to simplify the computations in the remainder. We remark that when using $\bal_R$,  problem \eqref{P1} is equivalent to the usual ratio cut problem.}.
 A variety of other balance terms have appeared in the literature in the context of two-class and multiclass clustering \cite{BLasymetric, pro:HeinSetzer11TightCheeger}. We refer to a pair $\{Y,\Y^c\}$ that solves \eqref{P1} as an \emph{optimal balanced cut of the graph}. Note that a given graph $\G = (\X,W)$ may have several optimal balanced cuts (although generically the optimal cut is unique).

We are also interested in multiclass balance cuts.  Specifically, in order to partition the set $X$ into $K \geq 3$ clusters, we consider the following ratio cut functional:
\begin{align}\label{eq:multi1:Body}
\underset{(\Y_1,\ldots,\Y_K)}{\text{Minimize}} \quad \sum^{K}_{k=1} \; \frac{ \cut(\Y_k,\Y^c_k) } { |\Y_k| },  \quad \Y_k \cap \Y_s = \emptyset \quad \text{if} \quad r \neq s, \quad \bigcup_{k=1}^{K} \Y_k = \X. 
\end{align}
 
 \subsection{Continuum partitioning} \label{cp1}

Given a bounded and connected open domain $D \subset \R^d$ and a probability measure $\nu$ on $D$, with positive density $\rho > 0$, we define the class of balanced domain cut problems in an analogous way. A balanced domain-cut problem takes the form
\begin{equation}  \label{P2}
\text{Minimize \;\;\;  } \frac{ \cut_\rho(A,A^c) } { \bal_\rho(A,A^c) },  \qquad  A  \subset D \; \te{ with } 0 < \nu(A) < 1. 
\end{equation} 
where $A^c = D \backslash A$. 
Just as the graph cut term $\cut(\Y,\Y^c)$ in \eqref{P1} provides a weighted (by $W$) measure of the boundary between $\Y$ and $\Y^c,$ the cut term $\cut_\rho(A,A^c)$ for a domain denotes a $\rho^2-$weighted area of the boundary between the sets $A$ and $A^c$.  If $\partial_D A := \partial A \cap D$ (the boundary between $A$ and $A^c$) is a smooth curve (in 2d), surface (in 3d) or manifold (in 4d$+$) then we define
\begin{equation}\label{eq:cut_def}
 \cut_\rho(A,A^c) \;\; :=  \;\;  \int_{\partial_D A} \rho^{2}(x)  \; \rd S(x). 
\end{equation}
For our results and analysis we need the notion of continuum cut which is defined for sets with less regular boundary. We present the required notions of geometric measure theory and 
the rigorous and mathematically precise formulation of problem \eqref{P2} in
Subsection \ref{sec:TV}.

If $\rho(x) = 1$ then $\cut_\rho(A,A^c)$ simply corresponds to arc-length (in 2d) or surface area (in 3d). 
In the general case, the presence of $\rho^{2}(x)$ in \eqref{eq:cut_def} indicates that the regions of low density are easier to cut, so $\partial A$ has a tendency to pass through regions in $D$ of low density. As in the graph case, we consider balance terms
\begin{equation} \label{contbal}
\bal_\rho(A,A^c)= 2|A| |A^c| \qquad \te{and} \qquad \bal_\rho(A,A^c)= \min(|A|, |A^c|), 
\end{equation}
 which correspond to weighted continuous equivalents of the Ratio Cut and the Cheeger Cut. In the continuum setting $|A|$ stands for the total $\nu$-content of  the set $A$, that is,
\begin{equation} \label{mass}
|A|=\nu(A) = \int_A \rho(x) \; \rd x. 
\end{equation}
We refer to a pair $\{A,A^c\}$ that solves \eqref{P2} as an \emph{optimal balanced cut of the domain}. 
\medskip

The continuum equivalent of the multiway cut problem \eqref{eq:multi1:Body} reads
\begin{align}\label{eq:multi1}
\underset{(A_1,\ldots,A_R)}{\text{Minimize}} \quad \sum^{R}_{r=1} \; \frac{ \cut_\rho(A_r,A^c_r) } { |A_r| },  \quad A_r \cap A_s = \emptyset \quad \text{if} \quad r \neq s, \quad  \bigcup_{r=1}^{R} A_r = D.
\end{align}

\subsection{Consistency of partitioning of data clouds} \label{pdc}
We consider the sample $\X_n= \{ \x_1, \dots, \x_n \}$ consisting of i.i.d. random points drawn from an underlying ground-truth measure $\nu$. We assume that $D$ is a bounded, open set with Lipschitz boundary. Furthermore we assume that $\nu$ has continuous density $\rho$ and that 
$ \lambda \leq \rho \leq \Lambda$ on $D$. 

To extract the desired information about the point cloud one builds a graph by connecting the nearby points. More precisely consider a kernel $\eta : \R^d \to [0, \infty)$ to be radially symmetric, radially decreasing,  and decaying to zero sufficiently fast. We introduce a parameter $\veps$ which basically describes over which length scale the data points are connected.
%
We assign for $i,j \in \{1, \dots, n \}$ the weights by
\begin{equation} \label{edgew}
 w_{i,j} = \eta \left(\frac{ \x_i - \x_j}{\veps} \right).
\end{equation}

As more data points are available one takes smaller $\veps$ to obtain increased resolution. 
That is, one sets the length scale $\veps_n$ based on the number of available data points. 
We investigate under what scaling of $\veps_n$ on $n$ the optimal balanced cuts 
(that is minimizers of \eqref{P1}) of the graph $\G_n = (\X_n, W_n)$ 
converge towards optimal balanced cuts in the continuum setting (minimizers of \eqref{P2}).
On Figure \ref{fig:beans}, we illustrate partitioning a data cloud sampled from the uniform distribution for the given domain $D$.



\begin{figure}[ht]
    \centering
\subfigure[A sample of $n=120$ points.]{\includegraphics[width=0.38 \textwidth]{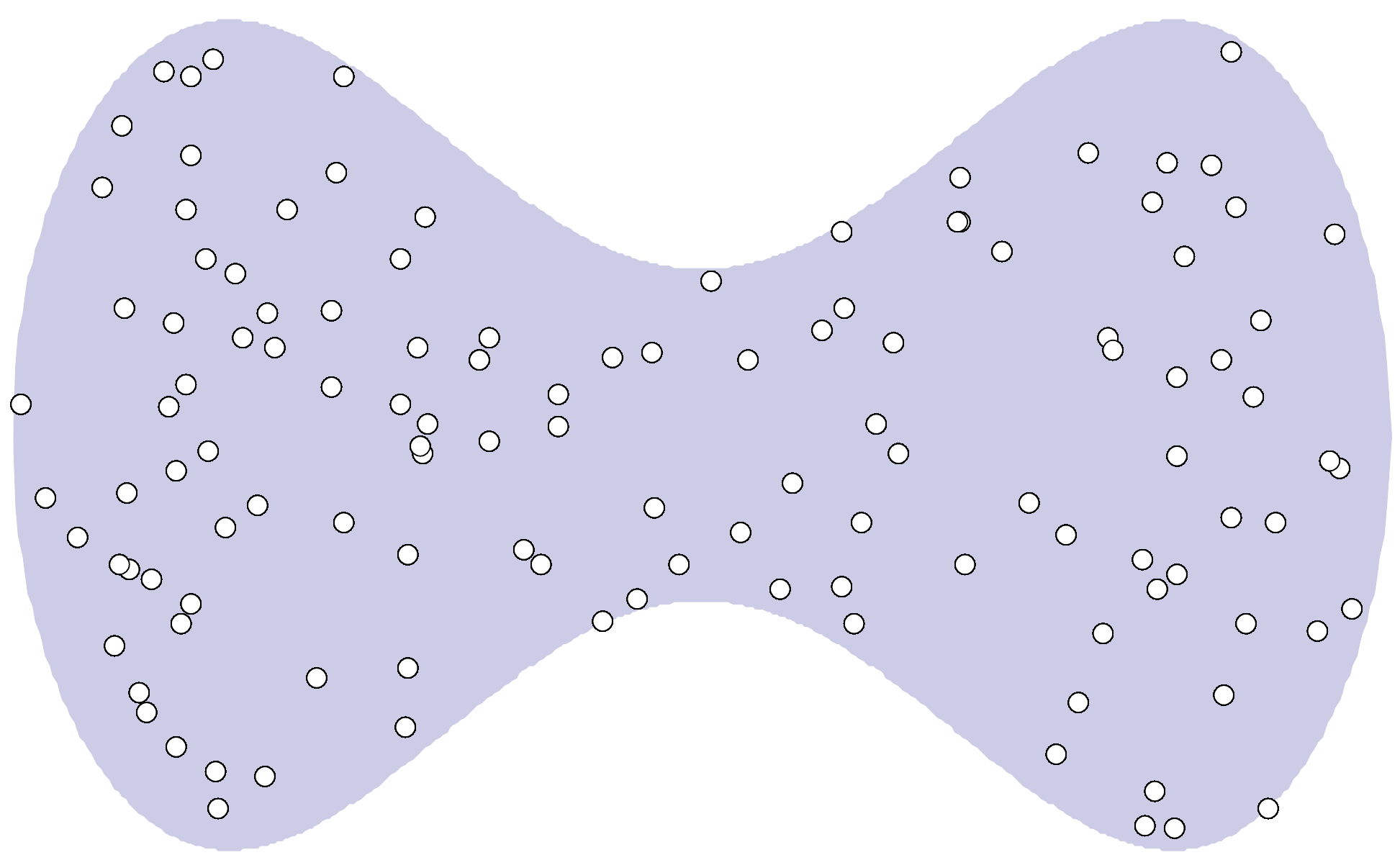}}  \hspace*{40pt} \nolinebreak
\subfigure[Geometric graph with $\veps=0.3$.]{\includegraphics[width=0.38 \textwidth]{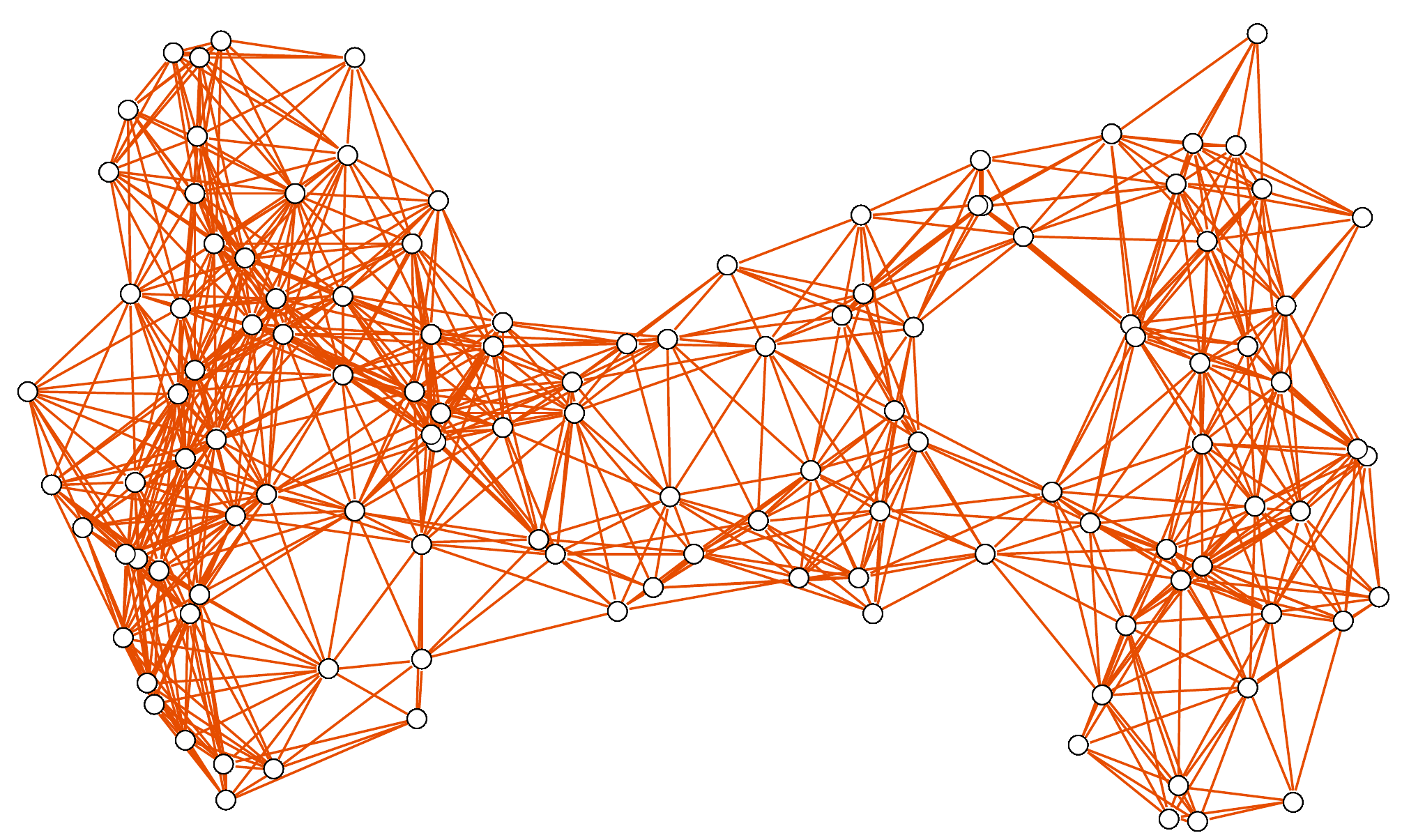}  } \\
        \medskip
\subfigure[Minimizer of Cheeger graph cut.]{\includegraphics[width=0.38 \textwidth]{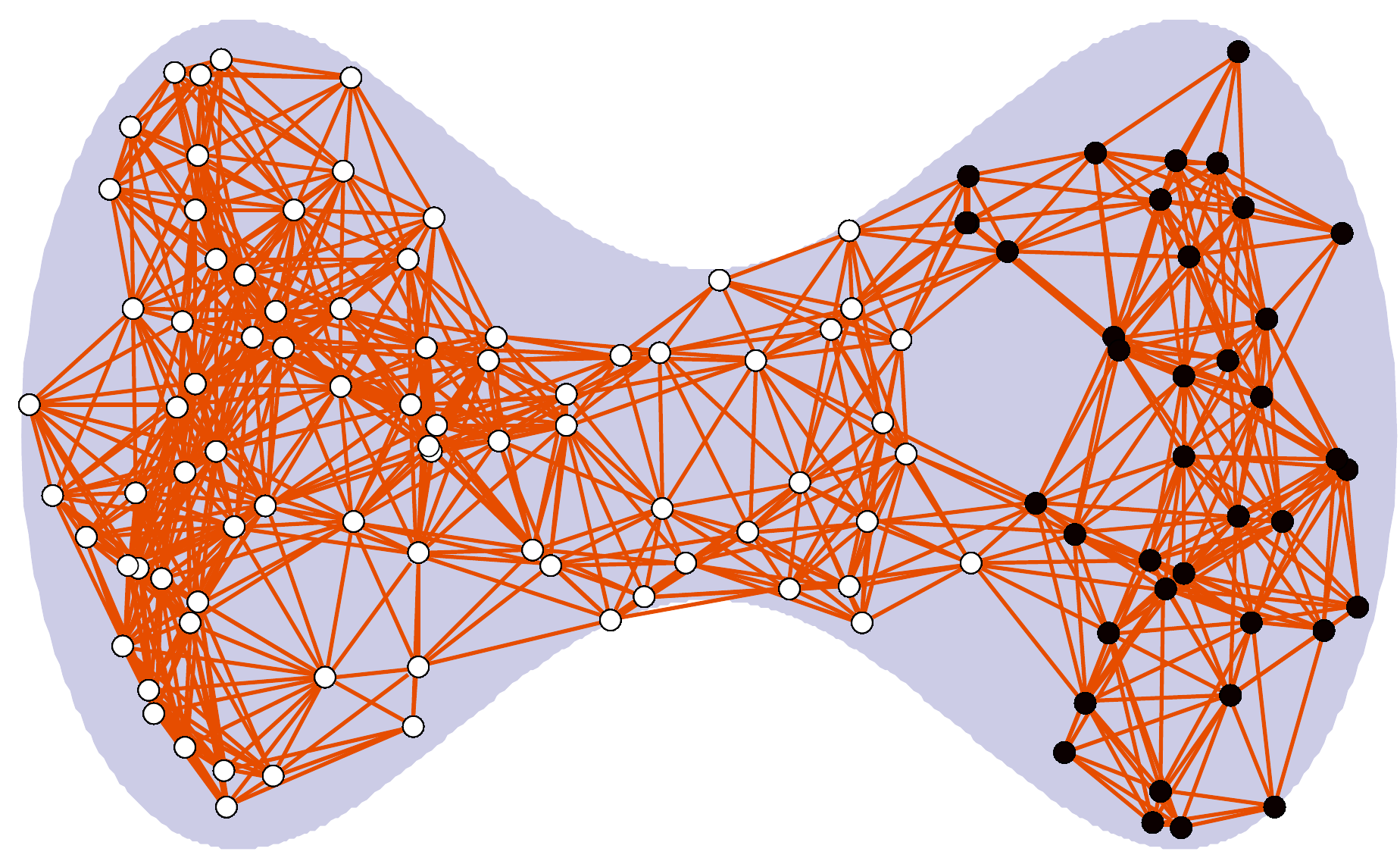}  }\hspace*{40pt} \nolinebreak
\subfigure[Minimizer of continuum Cheeger cut.]{\includegraphics[width=0.38 \textwidth]{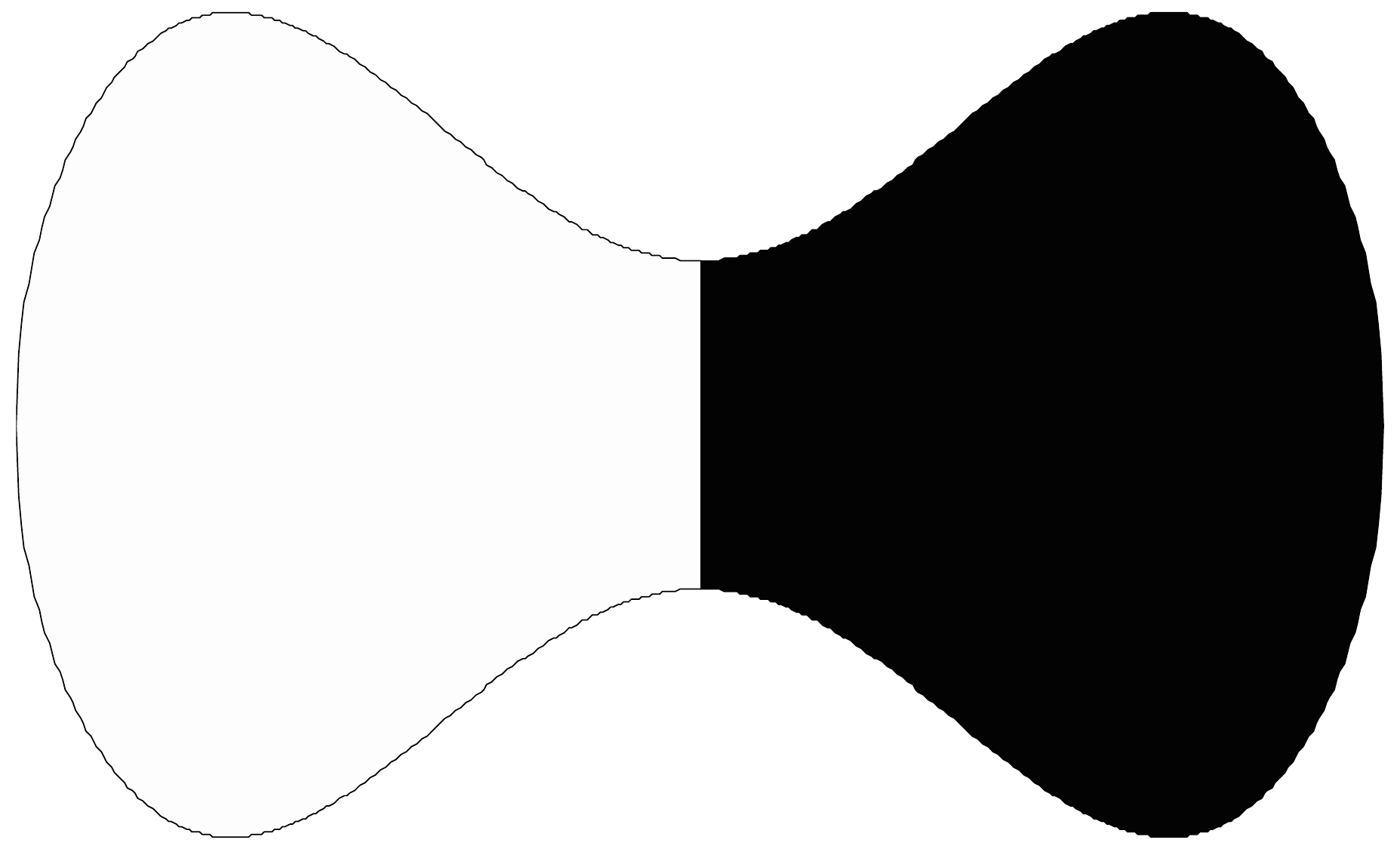}  }
\caption{Given the sample of Figure (a), graph is constructed using $\eta(z) = \one_{\{ |z|\leq 1\} }$  and $\veps = 0.3$, as illustrated on Figure (b). On Figure (c) we present the solution to the Cheeger graph-cut problem obtained using \cite{BLUV12}. A solution to the continuum Cheeger-cut problem is illustrated in Figure (d). }
\label{fig:beans}
\end{figure}

\medskip

\textbf{Informal statement of (a part of) the main results.} \emph{
Consider $d \geq 2$ and assume the continuum balanced cut \eqref{P2} has a unique minimizer $\{A,A^c\}$. Consider  $\veps_n>0$ such that $\lim_{n \to \infty} \veps_n=0$ and 
$$\lim_{n \to \infty}\frac{ (\log n)^{p_d}}{n^{1/d}} \frac{1}{\veps_n} =0,$$
where $p_d=1/d$ for $d \geq 3$ and $p_2=3/4$. Then almost surely the minimizers, $\{ \Y_n, \Y_n^c \}$, of 
the balanced cut \eqref{P1} of the graph $\mathcal G_n$ , converge to $\{A,A^c\}$.} Moreover, after appropriate rescaling, almost surely the minimum of problem \eqref{P1} converges to the minimum of \eqref{P2}.
\emph{
The result also holds for multiway cuts. That is the minimizers of \eqref{eq:multi1:Body} converge towards minimizers of \eqref{eq:multi1}.   }
\medskip

Let us make the notion of convergence of discrete partitions $\{\Y_n, \Y_n^c \}$ to continuum partitions  $\{A,A^c\}$ precise.
Let $\one_{Y_n} : \X_n \to \{0,1\}$ be the characteristic function of $Y_n$ on the set $X_n$. 
Let $\p_{i,1} = (\x_i, \one_{Y_n}(\x_i))$ and $\p_{i,2} = (\x_i, \one_{Y_n^c}(\x_i))$  for $i=1, \dots, n$ be the points on the graphs of $\one_{Y_n}$ and $\one_{Y_n^c}$ respectively.

%

Consider the probability measures on the graphs of $\one_{Y_n}$ and $\one_{Y_n^c}$: that is let
$ \gamma_{n,1}:= \frac{1}{n} \sum_{i=1}^n \delta_{\p_{i,1}} $ and $\quad  \gamma_{n,2} := \frac{1}{n} \sum_{i=1}^n \delta_{\p_{i,2}}$.
Let $\gamma$ be the push-forward of the measure $\nu$ to the graph of $\one_A$, that is  $\gamma = (Id \times \one_A)_\sharp \nu = \nu_{\llcorner  A} \times \delta_1 + \nu_{\llcorner  A^c} \times \delta_0$ where $Id$ is the identity mapping from $\R^d$ to $\R^d$ and
 $ \nu_{\llcorner  A}$ is the restriction of the measure $\nu$ to the set $A$ . 
We say that $\{Y_n, Y_n^c\}$ converge towards $\{A, A^c\}$ as $n \to \infty$ if there is a sequence of indices $I : \N \to \{1,2\}$ such that
\begin{equation} \label{altTLconv}
\gamma_{n, I(n)} \overset{w}{\rightharpoonup} \gamma\quad  \te{ as } n \to \infty,
\end{equation}
which is to be read as $\gamma_{n, I(n)}$ converges weakly to $\gamma$ (see \cite{DudleyBook} for the definition of weak convergence of probability measures.) In other words the convergence of discrete towards continuum partitions is defined as the weak convergence of graphs, considered as probability measures. 

In Section \ref{sec:trans} we discuss this topology in more detail and present a more general and conceptually clearer picture. In particular we point out that the weak convergence of measures on the space of graphs of functions is stronger than it may look and actually corresponds to $L^1$ convergence of functions. 

\begin{remark}[{Optimality of scaling of $\veps_n$} for $d \geq3$] \label{rem:intro}
If $d \geq 3$ then the rate presented in the statement above is 
sharp in terms of scaling. Namely for $D=(0,1)^d$, $\nu$ being the Lebesgue measure on $D$ and $\eta$ compactly supported, it is known from graph theory (see \cite{Goel,GuptaKumar, Penrose1}) that there exists a constant $c>0$ such that if $\veps_n < c \frac{(\log n)^{1/d}}{n^{1/d}}$ then the weighted graph associated to $(\X_n, W_n)$ is disconnected with high probability.
The resulting optimal discrete cuts have zero energy, but may be very far from the optimal continuum cuts.
While the above example demonstrates the optimality of our results, we caution that there 
may be settings relevant to machine learning in which the convergence of minimizers of appropriate functionals involving perimeter may hold even when $\frac{1}{n^{1/d}} \ll \veps_n  < c \frac{(\log n)^{1/d}}{n^{1/d}}$. 
\end{remark}
\begin{remark}In case $d=2$ the connectivity threshold for a random geometric graph is $\veps_n = c \frac{\log(n)^{1/2}}{n^{1/2}}$, which is below the rate for which we can establish the consistency of balanced cuts. Thus, an interesting open problem is to determine if the consistency results we present in this paper are still valid when the parameter $\veps_n$ is taken below the rate $\frac{\log(n)^{3/4}}{n^1/2}$ we obtained the proof for, but above the connectivity rate. 
In particular we are interested in determining if connectivity is the determining factor in order to obtain consistency of balance graph cuts. 
We numerically explore this problem in Section \ref{sec:numerics}. 
\label{Non-Optimald=2}
\end{remark}

We also remark that, despite the fact that for a general graphs the problems \eqref{P1} and \eqref{eq:multi1:Body} are NP hard, in practice when the graph is obtained by sampling from a measure $\nu$ as above, such minimization problems can be effectively approached \cite{BLUV12, BLUV13}. In fact, by choosing an appropriate initialization, the algorithms (see \cite{BLUV12, BLUV13}), 
give very good results in clustering real-world data. 
\nc

\nc
\subsection{Outline}
In Section \ref{sec:trans} we introduce the notion of convergence we use to bridge between discrete and continuum partitions. It relies on some of the notions of the theory of optimal transportation which we recall. Finally we recall results on optimal min-max matching which are needed in the proof of the convergence.
 In Section \ref{sec:CPTV} we study more carefully continuum partitioning \eqref{P2}.
We introduce the notion of total variation of functions on $D$ in Subsection \ref{sec:TV} and recall some of its basic properties. It enables us to introduce, in Subsection \ref{sec:cont_par2},
the general setting for problem \eqref{P2} where desirable properties such as lower semicontinuity and existence of minimizers hold.  
In Section \ref{sec:main}  we give the precise statement of the consistency result, both for the two-way cuts and the multiple-way cuts. Proving that minimizers of discrete balanced cuts converge to continuum balanced cuts relies on a notion of variational convergence known as $\Gamma$-convergence.
In Section \ref{sec:gamma-conv} we recall the definition of $\Gamma$ convergence and its basic properties. In Subsection \ref{sec:GGTV} we recall the results on $\Gamma$-convergence of graph total variation which provide the backbone for our result.
Section \ref{sec:proof1} contains the proof of the Theorem \ref{main} and Section \ref{sec:TheoremMulti-Class} the proof of Theorem \ref{main2}. 
Finally, in Section \ref{sec:numerics} we present numerical experiments which illustrate our results and we also investigate the issues related to Remark \ref{Non-Optimald=2}.  

\nc

 \section{From Discrete to Continuum} \label{sec:trans}
 
 For the two-class case, our main result shows that a sequence of partitions $\{\Y_n,\Y_n^c \}$ of the sample points $\X_n=\{\x_1, \ldots, \x_n \} \subset D$ converges toward a continuum partition $\{A,A^c\}$ of the domain $D$.
In this section we expand on the notion of convergence introduced in Subsection \ref{pdc} to compare the discrete and continuum partitions. We give an equivalent definition for such type of convergence  which turns out to be more useful for the computations in the remainder.

Associated to the partitions $\{\Y_n,\Y_n^c \}$ of $\X_n = \{ \x_1, \dots, \x_n\}$ there are characteristic functions of $\Y_n$ and $\Y_n^c$, namely $\one_{Y_n} : \X_n \to \{0,1\}$ and $\one_{Y_n}^c : \X_n \to \{0,1\}$. 
Let $\nu_n = \frac{1}{n} \sum_{i=1}^n \delta_{\x_i}$ be the empirical measures associated to $\X_n$
Note that $\one_{Y_n}$, $\one_{Y_n^c} \in L^1(\nu_n)$.
Likewise a continuum partition of $D$ by measurable sets $A$ and $A^c = D \backslash A$ can be described via the characteristic functions $\one_{A} : D \to \{0,1\}$ and  $\one_{A^c} : D \to \{0,1\}$. These too can be considered as $L^1$ functions, but with respect to the measure $\nu$ rather than $\nu_n$.

We compare partitions the $\{\Y_n,\Y_n^c \}$ and $\{A,A^c\}$ by comparing the associated characteristic functions. To do so, we need a way of comparing $L^1$ functions with respect to different measures. 
We follow the approach of \cite{GTS}. 
We denote by $\mathcal{B}(D)$ the Borel $\sigma$-algebra on $D$ and by $\mathcal{P}(D)$ the set of Borel probability measures on $D$.
The set of objects of our interest is
\[ TL^1(D) := \{ (\mu, f) \; : \:  \mu \in \mathcal P(D), \, f \in L^1(\mu) \}. \]
Note that $(\nu_n, \one_{Y_n})$ and  $(\nu, \one_{A})$ both belong to $TL^1$.
To compare functions defined with respect to different measures, say $(\mu,f)$ and $(\theta,g)$ in $TL^1$,  we need a way to say for  which $(x,y) \in \supp(\mu) \times \supp(\theta)$ should we compare $f(x)$ and $g(y)$. 
The notion of {\em coupling} (or \emph{transportation plan})  between $\mu$ and $\theta$, provides a way to do that. A coupling between $\mu, \theta \in \mathcal{P}(D)$ is a probability measure $\pi$ on the product space $D \times D$, such that the marginal on the first variable is $\mu$ and the marginal on the second variable is $\theta$. The set of couplings $\Gamma(\mu, \theta)$ is thus
\[ \Gamma(\mu, \theta) = \{ \pi \in \mathcal{P}(D \times D) \::\: (\forall U \in \mathcal{B}(D)) \;\:
\pi(U \times D) = \mu(U) \te{ and } \pi(D \times U) = \theta(U)\}. \]
For $(\mu,f)$ and $(\theta,g)$ in $TL^1(D)$ we define the distance
\begin{equation} \label{TL1}
d_{TL^1}((\mu,f), (\theta,g)) =
   \inf_{\pi \in \Gamma(\mu, \theta)} \iint_{D \times D} |x-y| + |f(x)-g(y)|  \rd\pi(x,y). 
\end{equation}
This is the distance that we use to compare $L^1$ functions with respect to different measures. To understand it better we focus on the case that one of the measures, say $\mu$, is absolutely continuous with respect to the Lebesgue measure, as this case is relevant for us when passing from discrete to continuum.
In this case the convergence in $TL^1$ space can be formulated in simpler ways using 
 transportation maps instead of couplings to match the measures.
Given a Borel map $T: D \rightarrow D$ and $\mu \in \mathcal{P}(D)$ the \emph{push-forward} of $\mu$ by $T$, denoted by $T_{\sharp} \mu \in \mathcal{P}(D)$ is given by:
\begin{equation*}
T_{\sharp} \mu(A):= \mu\left( T^{-1}(A) \right), \: A \in \mathfrak{B}(D).
\end{equation*}
A Borel map $T : D \rightarrow D$ is a \emph{transportation map} between the measures $\mu\in \mathcal{P}(D)$ and $\theta \in \mathcal{P}(D)$ if $\theta= T_\sharp \mu$. Associated to a transportation map $T$, there is a plan $\pi_T \in \Gamma(\mu , \theta)$ given by 
$\pi_T:= (\id \times T)_{\sharp}\mu$,
where $(\id \times T)(x) = \left(x, T(x) \right)$. 

We note that if $\theta = T_\sharp \mu$ then the following change of variables formula holds for any $f \in L^1(\theta)$
\begin{equation} \label{ChangeOfVariables}
\int_D f(y) d\theta(y) = \int_D f(T(x)) d \mu(x).
\end{equation}


In order to give the desired interpretation of convergence in $TL^1$ we also need the notion of a stagnating sequence of transportation maps. A sequence  $\left\{ T_n \right\}_{n \in \N}$  of transportation maps between $\mu$ and  $\left\{ \mu_n \right\}_{n \in \N}$ (i.e. ${T_n}_{\sharp} \mu = \mu_n$) is \emph{stagnating} if
\begin{equation}
  \int_{D} |x - T_n(x)| \rd \mu(x) \rightarrow 0 \qquad \te{as } n \to \infty.
  \label{def:StagnatingSeq}
\end{equation}
This notion is relevant to our considerations since for the measure $\nu$ and its empirical measures $\nu_n$ there exists (with probability one) a sequence of stagnating transportation maps ${T_n}_\sharp \nu = \nu_n$. The idea is that as $n \to \infty$ the mass from $\nu$ needs to be moved only very little to be matched with the mass of $\nu_n$. We make this precise in Proposition \ref{thm:InifinityTransportEstimate}

We now provide the desired interpretation of the convergence in $TL^1$, which is a  part of Proposition 3.12 in \cite{GTS}.
\begin{proposition} \label{prop:TLchar}
Consider a measure $\mu \in \mathcal{P}(D)$ which is absolutely continuous with respect to the Lebesgue measure.
Let $(\mu,f) \in TL^1(D)$ and let
 $\left\{ \left(\mu_n , f_n \right) \right\}_{n \in \N}$ be a sequence in $TL^1(D)$. The following statements are equivalent:
\begin{listi}
\item  $ \left(\mu_n , f_n \right)  \overset{{TL^1}}{\longrightarrow} (\mu, f)$ as $n \rightarrow \infty$.
\item  $\mu_n \overset{w}{\rightharpoonup}\mu$ and there exists a stagnating sequence of transportation maps ${T_n}_\sharp \mu = \mu_n$ such that:
\begin{equation}
\int_{D} \left| f(x) - f_n\left(T_n(x)\right) \right| d\mu(x) \rightarrow 0, \:  as \: n \rightarrow \infty.
\label{convergentTLpMap}
\end{equation}
\item $\mu_n \overset{w}{\rightharpoonup} \mu$ and for any stagnating sequence of transportation maps  ${T_n}_\sharp \mu = \mu_n$ convergence  \eqref{convergentTLpMap} holds.
\end{listi}
\label{EquivalenceTLp}
\end{proposition}
The previous proposition implies that in order to show that $(\mu_n, f_n)$ converges to $(\mu,f)$ in the $TL^1$-sense, it is enough to find
a sequence of stagnating transportation maps ${T_n}_\sharp \mu = \mu_n$ and then show the $L^1$
convergence of $f_n \circ T_n$ to $f$. 
An important feature of Proposition \ref{EquivalenceTLp} is that there is complete freedom on what sequence of transportation maps $\left\{T_n \right\}_{n \in \N}$ to take, as long as it is stagnating.
In particular this shows that if $\mu_n = \mu$ for all $n$ then the convergence in $TL^1$ is equivalent to convergence in $L^1$.
\medskip

To apply the above to our setting we need a stagnating sequence of transportation maps between $\nu$ and $\left\{ \nu_n \right\}_{n \in \N}$. The results on optimal transportation
provide such a sequence with precise information on the rate at which \eqref{def:StagnatingSeq} occurs.
For some of our considerations it is  useful to have  the control of $x - T_n(x)$ in the stronger $L^\infty$-norm , rather than in the $L^1$-norm. The following result of \cite{W8L8} provides such transportation maps with optimal scaling of the norm on $n$.
\begin{proposition} \label{thm:InifinityTransportEstimate}
Let $D$ be an open, connected and bounded subset of $\R^d$ which has Lipschitz boundary. Let $\nu$ be a probability measure on $D$ with density $\rho$ which is bounded from below and from above by positive constants. Let  $\x_1, \dots, \x_n, \dots$ be a sequence of independent random points  distributed on $D$ according to measure $\nu$ and let $\nu_n = \frac{1}{n} \sum_{i=1}^n \delta_{\x_i}$.
Then there is a constant $C>0$ such that with probability one there exists a sequence of transportation maps $\left\{ T_n \right\}_{n \in \N}$ from $\nu$ to $\nu_n$   ($T_{n \sharp} \nu = \nu_n$) and  such that: 
\begin{equation}
\label{}
\limsup_{n \rightarrow \infty}   \frac{n^{1/d} \|Id - T_n\|_\infty }{(\log n)^{p_d}}  \leq C,
\end{equation}
where the power $p_d$ is equal to $1/d$ if $d\geq 3$ and equal to $3/4$ if $d=2$.
\end{proposition}
\medskip

Having defined the $TL^1$-convergence for functions, we turn to defining the $TL^1$-convergence for partitions. When formalizing a notion of convergence for sequences of partitions $\{\Y_1^n,\dots,\Y_R^n\},$ we need to address the inherent ambiguity that arises from the fact that both $\{\Y_1^n,\dots,\Y_R^n \}$ and $\{\Y_{P(1)}^n,\dots,\Y_{P(R)}^n \}$ refer to the same partition for any permutation $P$ on the set $\left\{1, \dots, R \right\}$. Having the previous observation in mind, the convergence of partitions is defined in a natural way. 
\begin{definition}\label{def:TL1Partitions}
The sequence $ \left\{\Y_1^n,\dots,\Y_R^n \right\}_{n \in \N}$, where $\{\Y_1^n,\dots,\Y_R^n\}$ is a partition of $\X_n$, converges in the $TL^1$-sense to the partition $\{A_1, \dots, A_R\}$ of $D$, if there exists a sequence of permutations $\left\{ P_n \right\}_{n \in \N}$ of the set $\left\{  1, \dots, R\right\}$, such that for every $r \in \left\{1, \dots, R\right\}$, 
\[    \left(\nu_n, \one_{\Y_{P_n(r)}^n} \right)  \converges{TL^1} (\nu, \one_{A_r}) \qquad \te{as } n \to \infty. \]
\end{definition}
We note that the definition above is equivalent to the definition in \eqref{altTLconv} which we gave in Subsection \eqref{pdc} when discussing the main result. The equivalence follows from the fact that the $TL^1$ metric \eqref{TL1} can
be seen as the distance between the graphs of functions, considered as measures. Namely given $(\mu,f), (\theta,g) \in TL^1(D)$, let
$\Gamma_f = (\id \times f)_\sharp \mu$ and $\Gamma_g = (\id \times g)_\sharp \theta$ be the 
measures representing the graphs. Consider $d(\Gamma_f , \Gamma_g) := d_{TL^1}((\mu,f), (\theta,g)$. Proposition 3.3 in \cite{GTS} implies that this distance metrizes the weak convergence of measures on the family of graph measures. Therefore the convergence of partitions of Definition \ref{def:TL1Partitions} is equivalent to one given in \eqref{altTLconv}.
\medskip

We end this section by making some remarks about why the $TL^1$-metric is a suitable metric for considering consistency problems. On one hand if one considers a sequence of minimizers $\{ \Y_n, \Y_n^c \}$ of  the graph balanced cut \eqref{P1} the topology needs to be weak enough for the sequence of minimizers to be guaranteed to converge (at least along a subsequence). Mathematically speaking the 
topology needs to be weak enough for the sequence to be pre-compact. On the other hand the topology has to be strong enough for one to be able to conclude that the limit of a sequence of minimizes is a
minimizer of the continuum balanced cut energy. In Proposition \ref{prop:gamma} and Lemma \ref{com} we establish that the $TL^1$-metric satisfies both of the desired properties. 

Finally we point out that our approach from discrete to continuum can be interpreted as an extrapolation or extension approach, as opposed to restriction viewpoint. Namely when comparing $(\mu_n, f_n)$ and $(\mu,f)$ where $\mu_n$ is discrete and $\mu$ is absolutely continuous with respect to the Lebesgue measure we end up comparing 
two $L^1$ functions with respect to the Lebesgue measure, namely $f_n \circ T_n$ and $f$, in \eqref{convergentTLpMap}. Therefore $f_n \circ T_n$ used in Proposition \ref{prop:TLchar} can be seen as a continuum representative (extrapolation) of the discrete $f_n$. 
We think that this approach is more flexible and suitable for the task than the, perhaps more common, approach of comparing the discrete and continuum by restricting the continuum object to the discrete setting (this would correspond to considering $f|_{\supp(\mu_n)}$ and comparing it to $f_n$).

\section{Continuum partitioning: rigorous setting}
\label{sec:CPTV}

We first recall the general notion of (weighted) total variation and some notions of analysis and geometric measure theory.

\subsection{Total Variation} \label{sec:TV}

Let $D$ be an open and bounded domain in $\R^{d}$ with Lipschitz boundary and let $\rho: D \rightarrow (0,\infty)$ be a continuous density function. We let $\nu$ be the measure with density $\rho$. We assume that $\rho$ is bounded above and below by positive constants, that is, $ \lambda\leq \rho \leq\Lambda $ on $D$ for some $\Lambda \geq\lambda>0$. If needed, we consider an extension of $\rho$ to the whole $\R^d$ by setting $\rho(x)= \lambda$ for $x \in \R^d \setminus D$.
 This extension is a lower semi-continuous function and has the same lower and upper bounds that the original $\rho$ has. 
  
Given a function $u \in L^{1}(\nu)$, we define the weighted (by weight $\rho^2$) total variation of $u$ by:
\begin{equation}
TV(u; D) := \sup \left\{  \int_{D}  u(x) \mathrm{div}(\Phi(x)) \; \rd x :  \Phi(x) \in C^{1}_{c}(D: \R^d), \quad |\Phi(x)| \leq \rho^{2}(x) \right\}.
\label{TVWeighted}
\end{equation}
If $u$ is regular enough then the weighted total variation can be written as 
\begin{equation}
TV(u; D) = \int_D |\nabla u| \rho^2(x) \; \rd x.
\label{TVsmoothFunctions}
\end{equation}
Also, given that $\rho: D \rightarrow \R$ is continuous, if $u = \one_A$ is the characteristic function of a set $A \subseteq \R^d$ with $C^1$ boundary, then 
\begin{equation}
  TV(\one_{A} ; D ) = \int_{\partial A \cap D}   \rho^2(x) \;\rd \mathcal{H}^{d-1}(x),   
  \label{TVsmoothSets}
\end{equation}
where $\mathcal{H}^{d-1}$ represents the $(d-1)$-dimensional Hausdorff measure in $\R^d$.  
In case $\rho$ is a constant ( $\nu$ is the uniform distribution), the functional $TV(\cdot; D)$ reduces to a multiple of the classical total variation and in particular \eqref{TVsmoothSets} reduces to a multiple of the surface area of the portion of $\partial A$ contained in $D$.

Since $\rho$ is bounded above and below by positive constants, a function $u \in L^1(\nu)$ has finite weighted total variation if and only if it has finite classical total variation. Therefore, if $u \in L^1(\nu)$ with $TV(u; D) < \infty$, then $u$ is a BV function and hence it has a distributional derivative $Du$ which is a Radon measure (see Chapter 13 in \cite{Leoni}). We denote by $|Du|$ the total variation of the measure $Du$ and denote by $|Du|_{\rho^2}$ the measure determined by 
\begin{equation}
 \rd |Du|_{\rho^2}  = \rho^2(x) \rd |Du|.
 \label{WeightedDistributionalDerivativeVSSitributionalDerivative}
 \end{equation}
By Theorem 4.1 in \cite{Baldi}
\begin{equation}
 TV(u;D)  = |Du|_{\rho^2}(D) = \int_{D} \rho^2(x) \;\rd |Du|(x) .
 \end{equation}

A simple consequence of the definition of the weighted $TV$ is its lower semicontinuity with respect to $L^1$-convergence. More precisely, if $u_k \converges{L^1(\nu)} u$ then
\begin{equation} \label{TVlsc}
   TV(u ;D) \leq \liminf_{k \rightarrow \infty} TV(u_k;D).  
\end{equation}

Finally, for $u \in BV(D)$, the co-area formula 
$$ TV(u;D) = \int_{\R} TV(\one_{\left\{ u >  t \right\}}; D) \;\rd t,   $$
relates the weighted total variation of $u$ with the weighted total variation of its level sets. A proof of this formula can be found in \cite{BBF}. For a proof of the formula in case $\rho$ is constant see \cite{Leoni}.

In the remainder of the paper, we write $TV(u)$ instead of $TV(u;D)$ when the context is clear.

\subsection{Continuum partitioning} \label{sec:cont_par2}

We use the total variation to rigorously formulate the continuum partitioning problem \eqref{P2}. 
The precise definition of the $\cut_\rho(A, A^c)$ in functional in \eqref{P2} is
\[ \cut_\rho(A,A^c) = TV(\one_A; D), \]
where $TV(\one_A;D)$ is defined in \eqref{TVWeighted}. We note that $TV(\one_{A}; D)$ is equal to $TV(\one_{A^c}; D)$, and is the perimeter of the set $A$ in $D$ weighted by $\rho^2$. 

We also formulate the balance terms, defined by \eqref{contbal} and \eqref{mass}, using characteristic functions.
In fact, we start by extending the balance term to arbitrary functions  $u \in L^1(\nu)$: 
\begin{gather}
B_{ {\rm R}} (u)= \int_{D} |u(x) - \mean_\rho(u) | \rho(x) \;\rd x  \quad \text{and} \quad B_{ {\rm C}}(u)=  \min_{c \in \real}\int_{D} |u(x) - c | \rho(x) \;\rd x,  \label{balance}
\end{gather}
where $\mean_\rho(u)$ denotes the mean/expectation of $u(x)$ with respect to the measure $\rd \nu = \rho \rd x$. From here on, we use $B$ to represent either $B_{ {\rm R} }$ or $B_{ {\rm C} }$ depending on the context.
We have the relations:
\begin{gather} \label{frou}
 B_{ {\rm R} }(\1_{A}) = \bal_{ {\rm R}}(A,A^c), \;\;  B_{ {\rm C} }(\1_{A}) = \bal_{ {\rm C}}(A,A^c),
\end{gather}
for every measurable subset $A$ of $D$. We also consider \emph{normalized indicator functions} $\tilde{\1}_A$ given by
$$ \tilde{\1}_A:= \frac{ \one_{A}}{B(\one_A)} , \quad A \subseteq D,  $$
and consider the set
\begin{equation}
\mathrm{Ind}(D) :=  \left\{ u \in L^{1}(\nu):  u=\tilde{\1}_{A} 
   \text{ for some measurable set }  A \subseteq D \;\text{ with } \; B(\one_{A}) \neq 0   \right\}.  \label{ind}
\end{equation}
Then for $u = \tilde{\1}_A \in \mathrm{Ind}(D)$
\begin{equation}
TV(u) = TV(\tilde{\1}_{A}  )=TV\left( \frac{\1_{A}}{B(\one_{A})}\right)=  \frac{TV(\1_{A})}{B(\one_{A})}=\frac{2\cut(A,A^c)}{\bal(A,A^c) } .
\label{eqn:TvNormalizedCont}
\end{equation}

Thus, we deduce that problem \eqref{P2} is equivalent to :
\begin{equation}  \label{PfinalCont}
\text{Minimize \;\;\;  }  E(u) := \begin{cases} 
TV(u) & \text{ if } u  \in  \mathrm{Ind}(D) \\
+ \infty & \text{ otherwise.}
\end{cases} 
\end{equation}
Before we show that both the continuum ratio cut and Cheeger cut indeed have a minimizer  we need the following lemma:
\begin{lemma} \label{Bcont} 
\begin{itemize}
\item[(i)] The balance functions $B$ are continuous on $L^1(\nu)$. 
\item[(ii)] The set $\ind(D)$ is closed in $L^1(D)$. 
\end{itemize}
\end{lemma}
\begin{proof}
 Let us start by proving \rm (i). We first consider the balance term $B_{ {\rm C}}(u)$ that corresponds to the Cheeger Cut. Suppose that $u_k \to u$ in $L^1(\nu)$, and let $c_{k},c_{\infty}$ denote medians of $u_k$ and $u$ respectively. By definition, $c_k$ and $c$ satisfy
$$
c_{k} \in \underset{c \in \R}{\mathrm{argmin}} \; \int_{D} |u_{k}(x) - c| \; \rho(x) \; \rd x, \qquad c_{\infty} \in \underset{c \in \R}{\mathrm{argmin}} \; \int_{D} |u(x) - c| \; \rho(x) \; \rd x
$$ 
This implies that 
$$
\int |u_k(x)-c_k| \rho(x) \; \rd x  \le    \int|u_k(x)- c| \rho(x) \; \rd x$$
for any $c \in \R$, so that in particular we have
\begin{align*}
& \int |u_k-c_k| \rho(x) \; \rd x   -   \int|u-c_{\infty}| \; \rho(x) \; \rd x \\ &\le  \int |u_k-c_{\infty}| \drho  -   \int|u-c_{\infty}| \drho  \le  \int  |u_k -u| \drho = \|u_k - u\|_{L^1(\nu)}
\end{align*}
 Exchanging the role of $u_k$ and $u$ in this argument implies that the inequality
\begin{align*}
& \int |u-c_{\infty}| \drho -   \int|u_k-c_k| \drho \le  \int  |u -u_k| \drho \leq\|u_k - u\|_{L^1(\nu)}
\end{align*}
also holds. Combining these inequalities shows that $|B(u_k)-B(u) | \le \|u_k-u\|_{L_1(\nu)} \to 0$ as desired. Now consider the balance term $B_{ {\rm R}}(u)$ that corresponds to the ratio Cut. For the ratio cut, the inequality $||a|-|b|| \le |a-b|$ immediately implies
\begin{align*}
& \left| \int |u_k-\mean_\rho(u_k)| \drho -   \int|u-\mean_\rho(u)| \drho  \right|  \\
& \leq \int |u_k -u| \drho  + \int | \mean_\rho(u_k) - \mean_\rho(u)| \drho  \\
& \leq \int |u_k -u| \drho +  | \mean_\rho(u_k) - \mean_\rho(u)|. 
\end{align*}
Since $u_k \to u$ in $L^1(\nu)$ we have that $\mean_\rho(u_k) \to  \mean_\rho(u)$ and therefore $|B(u_k)-B(u) | \le \|u_k-u\|_{L_1(\nu)} +    | \mean_\rho(u_k) - \mean_\rho(u)|  \to 0$ as desired.

In order to prove (\rm{ii}) suppose that $\left\{u_k\right\}_{n \in \N}$ is a sequence in $\ind(D)$ converging in $L^1(\nu)$ to some $u \in L^1(\nu)$, we need to show that $u \in \ind(D)$. By (\rm{i}) we know that $B(u_k) \rightarrow B(u) $ as $k \rightarrow \infty$. Since $u_k \in \ind(D)$, in particular $B(u_k)=1$. Thus, $B(u)=1$. On the other hand,  $u_k \in \ind(D)$ implies that $u_k$ has the form $u_k=\alpha_k \1_{A_k}$. Since this is true for every $k$, in particular we must have that $u$ has the form $u= \alpha \1_{A}$for some real number $\alpha$ and some measurable subset $A$ of $D$. Finally, the fact that $B$ is 1-homogeneous implies that $1= B(u)= \alpha B(\1_{A}) $. In particular $B(\1_{A}) \not =0$ and $\alpha = \frac{1}{B(\1_{A})}$. Thus $u= \tilde{\1}_A$ with $B(\1_A) \not =0$ and hence $u \in \ind(D)$.
\end{proof}

\begin{lemma} \label{lem:exist_cont}
Let $D$ and $\nu$ be as stated at the beginning of this section. There exists a measurable set $A \subseteq D$ with $0< \nu(A) < 1$ such that $\tilde \one_A$ minimizes \eqref{PfinalCont}.
\end{lemma}
\begin{proof}
The statement follows by the direct method of the calculus of variations. 
Since the functional is bounded from below it suffices to show that it is lower semicontinuous with respect to the $L^1(\nu)$ norm and that a minimizing sequence is precompact in $L^1(\nu)$.
To show lower semi-continuity it is enough to consider a sequence  $u_n = \one_{A_n} \in \mathrm{Ind}(D)$ converging in $L^1(\nu)$ to $u \in L^1(\nu)$. 
From Lemma \ref{Bcont} it follows that $u \in \mathrm{Ind(D)}$ and hence $u = \tilde \one_A$ for some $A$ with $B(A) > 0$.
Therefore $\one_{A_n} \to \one_A$ as $n \to \infty$ in $L^1(\nu)$.
The lower semi-continuity then follows from the lower semi-continuity of the total variation \eqref{TVlsc}, the continuity of $B$ and 
the fact that since $B(\one_A)>0$, $1/B(\one_{A_n}) \to 1/B(\one_A) $ as $n \to \infty$. 

The pre-compactness of any minimizing sequence of \eqref{PfinalCont} follows directly from Theorem 5.1 in \cite{Baldi}, which completes the proof.
\end{proof}

\section{Assumptions and statements of main results.} \label{sec:main}

Here we present the precise hypotheses we use and  state precisely the main results of this paper.
Let $D$ be an open, bounded, connected subset of $\R^d$ with Lipschitz boundary, and let $\rho:D \rightarrow \R$ be a continuous density which is bounded below and above by positive constants, that is, for all $x \in D$
\begin{equation}
\label{eq:ConditionsRho}
 \lambda\leq \rho(x) \leq\Lambda 
\end{equation} 
for some $\Lambda \geq\lambda>0$. We let $\nu$ be the measure $\rd \nu = \rho \rd x$. Let $\eta: [0, \infty) \to [0, \infty)$ be a similarity kernel, that is, a function satisfying:
\begin{itemize} \addtolength{\itemsep}{3pt}
\item[\textbf{(K1)}] $\eta(0)>0$ and $\eta$ is continuous at $0$.  
\item[\textbf{(K2)}] $\eta$ is non-increasing.
\item[\textbf{(K3)}] $ \sigma_\eta:=  \int_{\Rd} \eta(|x|)|\langle x , e_{1} \rangle| \; \rd x <  \infty. $
\end{itemize}
We refer to the quantity $\sigma_\eta$ as the \textit{surface tension} associated to $\eta$.  These hypotheses on $\eta$ hold for the standard similarity functions used in clustering contexts, such as the Gaussian similarity function $\eta(r) = \exp(-r^2)$  and the proximity
 similarity kernel  (i.e. $\eta(r) = 1$ if $r\le1$ and $\eta(r) = 0$ otherwise). 
 
The main result of our paper is:
\begin{theorem}[Consistency of cuts]  \label{main}
Let domain $D$,  probability measure $\nu$ and kernel $\eta$ satisfy the conditions above. 
  Let $\veps_n$ denote any sequence of positive numbers converging to zero that satisfy
$$
\lim_{n \to 0}\frac{ (\log n)^{3/4}}{n^{1/2}} \frac{1}{\veps_n} =0 \quad (d = 2), \qquad  \lim_{n \to 0}\frac{ (\log n)^{1/d}}{n^{1/d}} \frac{1}{\veps_n} =0 \quad (d \ge 3).
$$
Let  $\{\x_j\}_{j\in\N}$ be an i.i.d. sequence of points in $D$ drawn from the density $\rho(x)$ and  let $\X_n = \{\x_{1},\ldots,\x_{n}\}$. Let $\G_n = (\X_n,W_n)$ denote the graph whose edge weights are
$$
w^n_{ij} := \eta\left(\frac{|\x_i-\x_j|}{\veps_n}\right) \qquad  1 \le i,j \le n.
$$
Finally, let $\{\Y_n^*, {\Y_n^*}^c \}$ denote any optimal balanced cut of $\G_n$ (solution of problem \eqref{P1}).   If $\{A^*,{A^*}^c\}$ is the unique optimal balanced cut of the domain $D$ (solution of problem \eqref{PfinalCont}) then with probability one the sequence $\{\Y^*_n,{\Y^*_n}^c \}$ converges to $\{A^*,{A^*}^c\}$ in the $TL^1$-sense. 
If there is more than one optimal continuum balanced cut \eqref{PfinalCont} then $\{\Y_n^*, {\Y_n^*}^c \}$ converges along a subsequence to an optimal continuum balanced cut.

Additionally, $\C_n$, the minimum balance cut of the graph  $\G_n$  (the minimum of \eqref{P1}), satisfies
\begin{equation}
\lim_{n \rightarrow \infty} \frac{\C_n}{n^2 \veps_n^{d+1}} = \sigma_\eta \C ,  
\label{ConvergenceCheegerConstants}
\end{equation}
where $\sigma_\eta $ is the surface tension associated to the kernel $\eta$ and $\C$ is the minimum of \eqref{P2}.

\end{theorem}
As we discussed in Remark \ref{rem:intro} for $d \geq 3$ the scaling of $\veps$  on $n$ is essentially the best possible.

The proof of Theorem \ref{main} relies on establishing a variational convergence of discrete balanced cuts to continuum balanced cuts called the $\Gamma$-convergence which we recall in Subsection \ref{sec:gamma-conv}. The proof utilizes the results obtained in \cite{GTS}, where the notion of $\Gamma$-convergence is introduced in the context of data analysis problems, and in particular the $\Gamma$-convergence of the graph total variation is considered.   
The $\Gamma$-convergence, together with a compactness result, provides sufficient conditions for the convergence of minimizers of a given family of functionals to the minimizers of a limiting functional. 

\begin{remark}
\label{rem:main}
A few remarks help clarify the hypotheses and conclusions of our main result. The scaling condition $\veps_n \gg (\log n)^{p_d}n^{-1/d}$ comes directly from the existence of transportation maps from Proposition \eqref{thm:InifinityTransportEstimate}. This means that $\veps_n$ must decay more slowly  than the maximal distance a point in $D$ has to travel to match its corresponding data point in $\X_n$. In other words, the similarity graph $\G_n$ must contain information on a larger scale than that on which the intrinsic randomness operates. 
Lastly, the conclusion of the theorem still holds if the partitions $\{\Y^*_{n}, {\Y_n^*}^c\}$ only approximate an optimal balanced cut, that is if the energies of $\{\Y^*_{n}, {\Y_n^*}^c\}$ satisfy 
$$\lim_{n \to \infty}   \left(\frac{\cut({\Y_n}^*,{\Y_n^*}^c)}{\bal(\Y_n^*, {\Y_n^*}^{c})} - \min_{\Y  \subsetneq \X_n}  \frac{\cut(\Y,\Y^c)}{\bal(\Y, \Y^c)} \right)= 0.$$
This important property follows from a general result on $\Gamma$-convergence which we recall in Proposition \ref{comp_gen}.
\end{remark}


We also establish the following multiclass equivalent to Theorem \ref{main}.

\begin{theorem}  \label{main2}
Let domain $D$, measure $\nu$, kernel $\eta$, sequence $\{\veps_n\}_{n \in \N}$,
sample points $\{\x_i\}_{i \in N}$, and graph $\G_n$
 satisfy the assumptions of Theorem \ref{main}.
Let $({\Y^*}^n_{1}, \ldots,  {\Y^*}^n_R )$ denote any optimal balanced cut of $\G_n$, that is a minimizer of
\eqref{eq:multi1:Body}. If $  (A^*_{1}, \ldots,  A^*_{R} )  $ is the unique optimal balanced cut of $D$
(i.e. minimizer of \eqref{eq:multi1})
 then with probability one the sequence
$ ({\Y^*}^n_{1}, \ldots,  {\Y^*}^n_R ) $ converges to $  (A^*_{1}, \ldots,  A^*_{R} )  $ in the $TL^1$-sense.
If the optimal continuum balanced cut is not unique then the convergence to a minimizer holds along subsequences.  
 Additionally, $\C_n$, the minimum of \eqref{eq:multi1:Body}, satisfies
$$ \lim_{n \rightarrow \infty} \frac{\C_n}{n^2 \veps_n^{d+1}} = \sigma_\eta \C ,  $$
where $\sigma_\eta $ is the surface tension associated to the kernel $\eta$ and $\C$ is the minimum of \eqref{eq:multi1}.
\end{theorem}
The proof of Theorem \ref{main2} involves modifying the geometric measure theoretical results from \cite{GTS}. This leads to a substantially longer and more technical proof than the proof of Theorem \ref{main}, but the overall spirit of the proof remains the same in the sense that the $\Gamma$-convergence plays the leading role. Finally, we remark that analogous observations to the ones presented in Remark \ref{rem:main} apply to Theorem \ref{main2}.

\section{Background on $\Gamma$-convergence}
\label{sec:gamma-conv}

We recall and discuss the notion of $\Gamma$-convergence. In the literature $\Gamma$-convergence is defined for deterministic functionals. Nevertheless, the objects we are interested in are random and thus we decided to introduce this notion of convergence in this non-deterministic setting.

Let $(X,d_X)$ be a metric space and let $(\Omega, \mathfrak{F}, \mathbb{P})$ be a probability space. Let  $F_n: X  \times \Omega \rightarrow[0,\infty] $ be a sequence of random functionals. 
\begin{definition} \label{def:Gamma}
The sequence $\left\{ F_n \right\}_{n \in \mathbb{N}} $  $ \Gamma$-converges with respect to metric  $d_X$ to the deterministic functional $F: X \rightarrow  [0, \infty]$ as $n \rightarrow \infty$ if with $\mathbb{P}$-probability one the following conditions hold simultaneously:
\begin{enumerate}
\item \textbf{Liminf inequality:} For every $x \in X$ and every sequence $\left\{ x_n \right\}_{n \in \mathbb{N}}$ converging to $x$,
\begin{equation*}
\liminf_{n \rightarrow \infty} F_n(x_n) \geq F(x),
\end{equation*}
\item  \textbf{Limsup inequality:} For every $x \in X$ there exists a sequence $\left\{ x_n \right\}_{n \in \N}$ converging to $x$ satisfying
\begin{equation*}
\limsup_{n \rightarrow \infty} F_n(x_n) \leq F(x).
\end{equation*}
\end{enumerate}
We say that $F$ is the $\Gamma$-limit of the sequence of functionals $\left\{F_n \right\}_{n \in \N}$ (with respect to the metric $d_X$).
\label{defGamma}
\end{definition}
\begin{remark}
In most situations one does not prove the limsup inequality for all $x \in X$ directly. Instead, one proves the inequality for all $x$ in a dense subset $X'$ of $X$  where it is somewhat easier to prove, and then deduce from this that the inequality holds for all $x \in X$.  To be more precise, suppose that the limsup inequality is true for every $x$ in a subset $X'$ of $X$ and the set $X'$ is such that for every $x \in X$ there exists a sequence $\left\{ x_k \right\}_{k \in \N}$ in $X'$  converging to $x$ and such that $F(x_k) \rightarrow F(x)$ as $k \rightarrow \infty$, then the limsup inequality is true for  every $x \in X$. It is enough to use a diagonal argument to deduce this claim. This property is not related to the randomness of the functionals in any way.

\label{DenseGamma}
\end{remark}

\begin{definition}
We say that  the sequence of nonnegative random functionals $\left\{ F_n\right\}_{n \in \mathbb{N}}$ satisfies the compactness property if with $\mathbb{P}$-probability one,  the following statement holds:
any sequence $\left\{x_n\right\}_{n\in \N}$ bounded in $X$ and for which
\begin{equation*}
\sup_{k \in \N} F_{n}(x_n) < +\infty,
\end{equation*} 
 is relatively compact in $X$.
\label{defCompac}
\end{definition}
\begin{remark}
The boundedness assumption of $\left\{ x_n \right\}_{n \in \N}$ in the previous definition is a necessary condition for relative compactness and so it is not restrictive.
\end{remark}

The notion of $\Gamma$-convergence is particularly useful when the functionals $\left\{ F_n \right\}_{n \in \N}$ satisfy the compactness property. This is because it guarantees that with $\mathbb{P}$-probability one,  minimizers (or approximate minimizers) of $F_n$  converge to minimizers of $F$ and it also guarantees convergence of the minimum energy of $F_n$ to the minimum energy of $F$ (this statement is made precise in the next proposition). This is the reason why $\Gamma$-convergence is said to be a variational type of convergence. The next proposition can be found in \cite{Braides, Dalmaso}. We present its proof  for completeness and for the benefit of the reader. We also want to highlight the way this type of convergence works as ultimately this is one of the essential tools used to prove the main theorems of this paper.

\begin{proposition} \label{comp_gen}
Let $F_n : X \times \Omega \rightarrow [0, \infty]$ be a sequence of random nonnegative functionals which are not identically equal to $+\infty$, satisfying the compactness property and $\Gamma$-converging to the deterministic functional $F: X \rightarrow [0, \infty]$ which is not identically equal to $+\infty$. If it is true that with $\mathbb{P}$-probability one, there is a bounded sequence $\left\{x_n  \right\}_{n \in \N}$ satisfying 
\begin{equation}
\lim_{n \rightarrow \infty} \left(F_n(x_n) - \inf_{x \in X} F_n(x)  \right)= 0 
\label{AlmostMin}
\end{equation}
Then, with $\mathbb{P}$-probability one the following statements hold
\begin{equation}
\lim_{n \rightarrow \infty} \inf_{x \in X}F_n(x) = \min_{x \in X} F(x),
\label{ConvMinEnergy}
\end{equation}
furthermore, every bounded sequence $\left\{ x_n \right\}_{n \in \N}$  in $X$ satisfying \eqref{AlmostMin} is relatively compact and each of its cluster points is a minimizer of $F$. In particular, if $F$ has a unique minimizer, then a bounded sequence $\left\{ x_n \right\}_{n \in \N}$ satisfying \eqref{AlmostMin} converges to the unique minimizer of $F$.
\label{VariationalGamma}
\end{proposition}
\begin{proof}
Consider $\Omega'$ a set with $\mathbb{P}$-probability one for which all the statements in the definition of $\Gamma$-convergence together with the statement of the compactness property hold. We also assume that for every $\omega \in \Omega'$, there exists a bounded sequence $\left\{ x_n \right\}_{n \in \N}$ satisfying \eqref{AlmostMin}. We fix such $\omega \in \Omega'$ and in particular we can assume that $F_n$ is deterministic for every $n \in \N$.

Let $\left\{x_n \right\}_{n \in \N}$ be a sequence as the one described above. Let $\tilde{x} \in X$ be arbitrary. By the limsup inequality we know that there exists a sequence $\left\{\tilde{x}_n  \right\}_{n \in \N}$  with $\tilde{x}_n \to \tilde{x}$ and such that
$$  \limsup_{n \rightarrow \infty} F_n(\tilde{x}_n )  \leq F(\tilde{x}).$$
By \ref{AlmostMin} we deduce that
\begin{equation}
  \limsup_{n \rightarrow \infty}  F(x_n) = \limsup_{n \rightarrow \infty} \inf_{x \in X} F_n(x) \leq  \limsup_{n \rightarrow \infty} F_n(\tilde{x}_n )  \leq F(\tilde{x}),
  \label{AuxGamma1}  
  \end{equation}
and since $\tilde{x}$ was arbitrary we conclude that
\begin{equation}
 \limsup_{n \rightarrow \infty}  F_n(x_n) \leq \inf_{x \in X} F(x) . 
 \label{AuxGamma2}
\end{equation}
The fact that $F$ is not identically equal to $+\infty$ implies that the term on the right hand side of the previous expression is finite and thus $ \limsup_{n \rightarrow \infty}  F_n(x_n) < +\infty$. Since the sequence $\left\{x_n \right\}_{n \in \N}$ was assumed bounded, we conclude from the compactness property for the sequence of functionals $\left\{F_n \right\}_{n \in \N}$ that $\left\{x_n \right\}_{n \in \N}$ is relatively compact. 

Now let $x^*$ be any accumulation point of the sequence $\left\{ x_n \right\}_{n \in \N}$ ( we know there exists at least one due to compactness), we want to show that $x^*$ is a minimizer of $F$. Working along subsequences, we can assume without the loss of generality that $x_{n} \to x^*$. By the liminf inequality, we deduce that
\begin{equation}
 \inf_{x \in X} F(x) \leq F(x^*) \leq \liminf_{n \rightarrow \infty}F(x_{n}).
 \label{AuxGamma3} 
\end{equation} 
The previous inequality and \eqref{AuxGamma1} imply that 
$$  F(x^*) \leq F(\tilde{x}),  $$ 
where $\tilde{x}$ is arbitrary. Thus, $x^*$ is a minimizer of $F$ and in particular $\inf_{x \in X} F(x) = \min_{x \in X} F(x)$. 
Finally, to establish \eqref{ConvMinEnergy} note that this follows from \eqref{AuxGamma2} and \eqref{AuxGamma3}.
\end{proof}

\subsection{$\Gamma$-convergence of graph total variation} \label{sec:GGTV}

Of fundamental importance in obtaining our results is the $\Gamma$-convergence of the \textit{graph total variation} proved in \cite{GTS}. Let us describe this functional and also let us state the results we use. Given a point cloud $\X_n:= \left\{\x_1, \dots, \x_n \right\} \subseteq D$ where $D$ is a domain in $\R^d$, we denote by $GTV_{n , \veps_n}: L^1(\nu_n) \rightarrow [0,\infty]$ the functional:
\begin{equation}
GTV_{n , \veps_n}(u_n):= \frac{1}{n^2\veps_n^{d+1} } \sum_{i,j=1}^{n}\eta\left( \frac{|\x_i- \x_j|}{\veps_n}\right)  | u_n(\x_i)- u_n (\x_j)|, 
\label{def:GraphTV}
\end{equation}
where $\eta$ is a Kernel satisfying conditions \textbf{(K1)-(K3)}. The connection of the functional $GTV_{n, \veps_n}$ to problem \eqref{P1} is the following: if  $\Y_n$  is a subset of $\X_n$, then the graph total variation of the indicator function $\1_{\Y_n}$ is equal to a rescaled version of the graph cut of $\Y_n$, that is, 
$$   GTV_{n, \veps_n}(\1_{\Y_n}) =\frac{2 \cut(\Y_n,\Y_n^c)}{n^2 \veps_n^{d+1}} .$$
Now we present the results obtained in \cite{GTS}.

\begin{theorem}[Theorem 1.1 in \cite{GTS} ]
Let domain $D$, measure $\nu$, kernel $\eta$, sequence $\{\veps_n\}_{n \in \N}$,
sample points $\{\x_i\}_{i \in N}$, and graph $\G_n$
 satisfy the assumptions of Theorem \ref{main}. Then, $ GTV_{n,\veps_n}$, defined by \eqref{def:GraphTV}, $\Gamma$-converge to $\sigma_\eta TV$ as $n \rightarrow \infty$ in the $TL^1$ sense, where $\sigma_\eta$ is the surface tension associated to the kernel $\eta$  (see condition (\textbf{K3}))
 and $TV$ is the weighted (by $\rho^2$) total variation functional defined in \eqref{TVWeighted}.
 \label{DiscreteGamma}
\end{theorem}

Moreover, we have the following compactness result.
\begin{theorem}[Theorem 1.2 in \cite{GTS}]Under the same hypothesis of Theorem 1.1 in \cite{GTS}, the sequence of functionals $\left\{GTV_{n, \veps_n} \right\}_{n \in \N}$ satisfies the compactness property. Namely, if a sequence $\left\{u_n \right\}_{n \in \N}$  with $u_n \in L^1(\nu_n)$ satisfies 
\begin{equation*}
\sup_{n \in \N}  \|u_{n}\|_{L^1(\nu_{n})} < \infty,
\end{equation*}
and
\begin{equation*}
\sup_{n \in \N}  GTV_{n , \veps_n}(u_{n}) < \infty,
\end{equation*}
then $\{ u_n \}_{n \in N}$ is $TL^1$-relatively compact. 
\label{compact}
\end{theorem}

Finally, Corollary 1.3 in \cite{GTS} allows us to restrict the functionals $GTV_{n , \veps_n}$ and $TV$ to characteristic functions of sets and still obtain $\Gamma$-convergence. 

\begin{theorem}[Corollary 1.3 in \cite{GTS}]
Under the assumptions of Theorem 1.1 in \cite{GTS},  with probability one the following statement holds: for every $A \subseteq D$ measurable, there exists a sequence of sets $\left\{ \Y_n \right\}_{n \in \N}$ with $\Y_n \subseteq \X_n$ such that,
$$  \one_{\Y_n } \converges{TL^1} \one_{A}  $$
and
$$  \limsup_{n \rightarrow \infty} GTV_{n, \veps_n}(\one_{\Y_n})  \leq \sigma_\eta TV(\one_{A}). $$
\label{CorollaryDiscreteGamma}
\end{theorem}


\section{Consistency of two-way balanced cuts} \label{sec:proof1}

Here we prove  Theorem \ref{main}. 

\subsection{Outline of the proof}
\label{OutlineTwo-Class}

Before proving that $\{\Y_n^*,{\Y_n^*}^c\}$ converges to $\{A^*,{A^*}^c\}$ in the sense of Definition \ref{def:TL1Partitions}, we first pause to outline the main ideas. Rather than directly working with the sets $\Y_n^*$ and ${\Y_n^*}^c$, it proved easier to work with their indicator functions $  \1_{\Y_n^*}(x)$ and $\1_{{\Y_n^*}^c}(x)$ instead.  We first show, by an explicit construction in Subsection \ref{subsec:ApproxFuncTwo-class}, that
\begin{equation} \label{l1-min}
u_n^*:=  \tilde{\1}_{\Y_n^*}(x),\;\;  u_n^{**}:= \tilde{\1}_{{\Y_n^*}^c}(x)  \qquad \text{ minimize } \quad   E_n(u_n) \;\; \text{ over all } \;\; u_n \in L^1(\nu_n),
\end{equation}
where $E_n$ denotes a suitable objective function defined on $L^1(\nu_n)$, the set of functions defined over $\X_n$. Each function $\tilde{\1}_{\Y_n}(x) := \alpha_n \1_{\Y_n}(x)$ is simply a rescaled version of the original indicator function (for some explicit coefficient $\alpha_n$ that we will define later). Similarly, in Subsection \ref{sec:cont_par2} we showed that the normalized indicator functions
\begin{equation} \label{l1-min2}
u^*:=\tilde{\1}_{A^*}(x),\;\; u^{**}:= \tilde{\1}_{{A^*}^c}(x)  \quad \text{ minimize } \quad   E(u) \;\; \text{ over all }\;\; u \in L^1(\nu),
\end{equation}
where $E$ is defined by \eqref{PfinalCont} . 

In Subsection \ref{subsec:gamma} we show that the approximating functionals $E_n$ $\Gamma$-converge to $\sigma_\eta E$ in the $TL^1$-sense. In Lemma \ref{com}  we establish that  $u_n^*$ and  $u_n^{**}$ exhibit the required compactness.  Thus, they must converge toward the normalized indicator functions $\tilde{\1}_{A^*}(x)$ and  $\tilde{\1}_{{A^*}^c}(x)$ up to relabeling  (see Proposition \ref{comp_gen}). 
If $\{A^*, A^{*c}\}$ is the unique minimizer, the convergence of the whole sequence follows. 
The convergence of the partition $\{\Y_n^*,{\Y_n^*}^c\}$ toward the partition $\{A^*,{A^*}^c\}$ in the sense of Definition \ref{def:TL1Partitions} is a direct consequence.
The convergence \eqref{ConvergenceCheegerConstants} follows from \eqref{ConvMinEnergy} in Proposition \ref{comp_gen}. 

\subsection{Functional description of discrete cuts}
\label{subsec:ApproxFuncTwo-class}

We introduce functionals that  describe the discrete ratio and Cheeger cuts in terms of functions on $\X_n$, rather than in terms of subsets of $\X_n$. This mirrors the description of continuum partitions provided in 
Subsection \ref{sec:cont_par2}. 
For $u_n \in L^1(\nu_n)$, we start by defining
\begin{gather}
B^n_{ {\rm R}} (u_n):= \frac{1}{n} \sum_{i=1}^n |u_n(\x_i) - \mean_n(u_n) | \quad \text{and} \quad B^n_{ {\rm C}}(u_n):=  \min_{c \in \real} \frac{1}{n} \sum_{i=1}^n  |u_n(\x_i) - c |.  \label{balancen}
\end{gather}
Here $\mean_n(u_n) = \frac{1}{n} \sum_{i=1}^n u_n(\x_i)$. A straightforward computation shows that for $\Y_n \subseteq \X_n$ 
\begin{gather} \label{froun}
B^n_{ {\rm R} }(\1_{\Y_n}) = \bal_{ {\rm R}}(\Y_n,\Y^c_n), \;\;  B^n_{ {\rm C} }(\1_{\Y_n}) =  \bal_{ {\rm C}}(\Y_n,\Y^c_n).
\end{gather}
From here on we write $B_n$ to represent either $B^n_{ {\rm R} }$ or $B^n_{ {\rm C} }$ depending on the context.

Instead of defining $E_{n}(u_n)$ simply as the ratio $GTV_{n, \veps_n}(u_n)/B_n(u_n),$ which is the direct analogue of \eqref{P1}, it proves easier to work with suitably normalized indicator functions.
Given $\Y_n \subseteq \X_n$ with $B_n(\one_{\Y_n}) \neq 0$, the  \emph{normalized indicator function} $\tilde{\1}_{\Y_n}(x)$ is defined by
$$
\tilde{\1}_{\Y_n}(x) =  \one_{\Y_n}(x)/B^n_{ {\rm C} }(\one_{\Y_n})  \text{\quad or \quad } \tilde{\1}_{\Y_n}(x) =  \one_{\Y_n}(x)/B^n_{ {\rm R} }(\one_{\Y_n}).
$$
Note that $B_n(\tilde{\1}_A)=1$. We also restrict the minimization of $E_{n}(u)$ to the set 
\begin{equation}
\mathrm{Ind}_n(D) :=  \{ u_n \in L^{1}(\nu_n):  u_n=\tilde{\1}_{\Y_n} 
   \text{ for some }  \Y_n \subseteq \X_n \;\text{ with } \; B_n(\one_{\Y_n}) \neq 0   \}.  \label{ind_n}
\end{equation}

Now, suppose that $u_n \in \mathrm{Ind}_n(D)$, i.e.  that $u_n=\tilde{\1}_{\Y_n}$ for some set $\Y_n$ with $B_n(\1_{\Y_n})>0$. Using \eqref{frou} together with the fact that $GTV_{n,\veps_n}$ (defined in \eqref{def:GraphTV}) is one-homogeneous implies, as in \eqref{eqn:TvNormalizedCont}
\begin{equation}
GTV_{n, \veps_n}(u_n) =\frac{2}{n^2 \veps_n^{d+1} } \frac{\cut(\Y_n,\Y^c_n)}{\bal(\Y_n,\Y^c_n) } .
\label{eqn:TvNormalizedDiscrete}
\end{equation}
Thus, minimizing $GTV_{n, \veps_n}$ over all $u_n \in \mathrm{Ind}_{n}(D)$ is equivalent to the balanced graph-cut problem \eqref{P1} on the graph $\G_n = (\X_n,W_n)$ constructed from the first $n$ data points. We have therefore arrived at our  destination, i.e. a proper reformulation of \eqref{P1} defined over functions $u_n \in L^1(\nu_n)$ instead of subsets of $\X_n$:
\begin{equation}  \label{Pfinal}
\text{Minimize \;\;\;  }  E_n(u_n) := \begin{cases} 
GTV_{n, \veps_n}(u_n) & \text{ if } u_n  \in  \mathrm{Ind}_{n}(D) \\
+ \infty & \text{ otherwise.}
\end{cases} \\
\end{equation}

\subsection{$\Gamma$-Convergence}
\label{subsec:gamma}

\begin{proposition}{\rm ($\Gamma$-Convergence)}\label{prop:gamma}
Let domain $D$, measure $\nu$, kernel $\eta$, sequence $\{\veps_n\}_{n \in \N}$,
sample points $\{\x_i\}_{i \in N}$, and graph $\G_n$
 satisfy the assumptions of Theorem \ref{main}.
Let $E_n$ be as defined in \eqref{Pfinal} and $E$ as in \eqref{PfinalCont}.
Then 
\[ E_n \overset{\Gamma}{\longrightarrow} \sigma_n E \quad \te{ with respect to } TL^1 \te{ metric as }  n \to \infty \]
where $\sigma_\eta$ is the surface tension defined in assumption (\textbf{K3}). That is
\begin{enumerate}
\item For {any} $u \in L^1(\nu)$ and {  any} sequence $\left\{ u_{n} \right\}_{n \in \N}$ with $u_n \in L^1(\nu_n)$ that converges to $u$ in $TL^1$, 
\begin{equation}\label{eq:liminf}
\sigma_\eta E(u) \leq \liminf_{n \to \infty} E_{n}(u_n).
\end{equation} 
\item For {  any} $u \in  L^1(\nu)$ there exists {  at least one} sequence $\left\{ u_{n} \right\}_{n \in \N}$ with $u_n \in L^1(\nu_n)$ which converges to $u$ in $TL^1$ and also satisfies
\begin{equation}\label{eq:limsup}
\limsup_{n \to \infty} E_{n}(u_n) \leq \sigma_\eta E(u).
\end{equation} 
\end{enumerate}

\end{proposition}
We leverage Theorem \ref{DiscreteGamma} to prove this claim. We first need a preliminary lemma which allows us to handle the presence of the additional balance terms in \eqref{Pfinal} and \eqref{PfinalCont}.

\begin{lemma} \label{continuity} 
\begin{itemize}
\item[(i)] If $\left\{ u_n \right\}_{n \in \N}$ is a sequence with $u_n \in L^1(\nu_n)$ and $u_n \converges{TL^1} u$ for some $u \in L^1(\nu)$, then $B_n(u_n) \rightarrow B(u)$.
\item[(ii)] If $u_n = \tilde{\1}_{\Y_n}$, where $\Y_n \subset \X_n$, converges to $u = \tilde{\1}_{A}$ in the $TL^1$-sense, then $\one_{\Y_n} $ converges to $\one_{A}$ in the $TL^1$-sense.
\end{itemize}
\end{lemma}
\begin{proof}
To prove (\rm{i}), suppose that $u_n \in L^1(\nu_n)$ and that $u_n \converges{TL^1} u$. Let us consider $\left\{T_n \right\}_{n \in \N}$ a stagnating sequence of transportation maps between $\nu$ and $\left\{\nu_n \right\}_{n \in \N}$. Then, we have $u_n\circ T_n \converges{L^1(\nu)}  u $ and thus by (\rm{i}), we have that $B(u_n\circ T_n) \rightarrow B(u)$. To conclude the proof we notice that $B(u_n\circ T_n)=B_n(u_n)$ for every $n$. In fact, by the change of variables \eqref{ChangeOfVariables} we have that for every $c \in \R$ 
\begin{equation}
 \int_{D}|u_n(x)- c| \rd\nu_n(x) = \int_{D}| u_n\circ T_n(x) -c| \rd\nu(x)
 \label{Lemma:ChangeofVar} 
 \end{equation}
In particular we have $B_C^n(u_n)= B_C(u_n\circ T_n)$. Applying the change of variables \eqref{ChangeOfVariables}, we obtain $\mean_n(u_n) = \mean_\rho(u_n \circ T_n)$ and combining with \eqref{Lemma:ChangeofVar} we deduce that $B_R^n(u_n)= B_R(u_n\circ T_n)$. 

The proof of ${(\rm{ii})}$ is straightforward.
\end{proof}

Now we turn to the proof or Proposition \ref{prop:gamma}.
\begin{proof}
\textbf{ Liminf Inequality. }For arbitrary $u \in L^1(\nu)$ and arbitrary sequence $\left\{ u_n \right\}_{n \in \N}$ with $u_n \in L^1(\nu_n)$ and with $u_n \converges{TL^1}u$, we need to show that
$$  \liminf_{n \rightarrow \infty}E_n(u_n) \geq \sigma_\eta E(u). $$ First assume that $u \in \ind(D)$. In particular $E(u)=TV(u)$. Now, note that working along a subsequence we can assume that the liminf is actually a limit and that this limit is finite (otherwise the inequality would be trivially satisfied). This implies that for all $n$ large enough we have $E(u_n)<+\infty$, which in particular implies that $E_n(u_n)=GTV_{n,\veps_n}(u_n)$. Theorem \ref{DiscreteGamma} then implies that
$$ \liminf_{n \rightarrow \infty}E_n(u_n) = \liminf_{n \rightarrow \infty}GTV_{n, \veps_n}(u_n) \geq \sigma_\eta TV(u) = \sigma_\eta E(u).   $$
Now let as assume that $u \not \in \ind(D)$. Let us consider a stagnating sequence of transportation maps $\left\{ T_n\right\}_{n \in \N}$ between $\left\{\nu_n \right\}_{n \in \N}$ and $\nu$. Since $u_n \converges{TL^1} u$ then $u_n\circ T_n \converges{L^1(\nu)} u$. By Lemma \ref{continuity} , the set $\ind(D)$ is a closed subset of $L^1(\nu)$. We conclude that  $u_n\circ T_n \not \in \ind(D)$ for all large enough $n$. From the proof of Lemma \ref{continuity} we know that $B_n(u_n)= B(u_n\circ T_n) $ and from this fact, it is straightforward to show that $u_n \circ T_n  \not \in \ind(D)$ if and only if $u_n \not \in \ind_n(D) $. Hence, $u_n \not \in \ind_n(D)$ for all large enough $n$ and in particular $\liminf_{n \in \N} E_n(u_n) = + \infty$ which implies that the desired inequality holds in this case.

\textbf{Limsup Inequality.} We now consider $u \in L^1(\nu)$. We want to show that there exists a sequence $\left\{u_n \right\}_{n \in \N}$ with $u_n \in L^1(\nu_n)$ such that
$$  \limsup_{n \rightarrow \infty} E_n(u_n) \leq \sigma_\eta E(u). $$
Let us start by assuming that $u \not \in \ind(D)$. In this case $E(u)=+\infty$. From Theorem \ref{DiscreteGamma} we know there exists at least one sequence $\left\{u_n \right\}_{n \in \N }$ with $u_n \in L^1(\nu_n)$ such that $u_n \converges{TL^1} u$. Since $E(u) = +\infty$, the inequality is trivially satisfied in this case.

On the other hand, if $u \in \ind(D)$, we know that $u = \tilde{\1}_{A}$ for some measurable subset $A$ of $D$ with $B(\1_A) \not = 0$. By Theorem \ref{CorollaryDiscreteGamma}, there exists a sequence $\left\{ \Y_n \right\}_{n \in \N}$ with $\Y_n \subseteq \X_n$, satisfying $\1_{\Y_n} \converges{TL^1} \1_{A}$ and 
\begin{equation}
  \limsup_{n \rightarrow \infty} GTV_{n , \veps_n}( \1_{\Y_n})  \leq \sigma_\eta TV(\1_A).  
  \label{aux:prop2:0}
\end{equation}
Since $\1_{\Y_n} \converges{TL^1} \1_{A}$ Lemma \ref{continuity} implies that 
\begin{equation}
B_n(\1_{\Y_n}) \rightarrow B(\1_{A}).
\label{aux:prop2}
\end{equation}
In particular $B_n(\1_{\Y_n}) \not =0 $ for all $n$ large enough, and thus we can consider the function $u_n := \tilde{\1}_{\Y_n} \in \ind_n(D)$. From \eqref{aux:prop2} it follows that $u_n \converges{TL^1} u$ and together with \eqref{aux:prop2:0} it follows that
$$ \limsup_{n \rightarrow \infty} GTV_{n,\veps_n}(u_n) = \limsup_{n \rightarrow \infty} \frac{1}{B_n(\Y_n)} GTV_{n,\veps_n}(\1_{\Y_n}) \leq \frac{1}{B(\1_A)} \sigma_\eta TV(\one_{A})= \sigma_\eta TV(u)     $$   
Since, $u_n \in \ind_n(D)$ for all $n$ large enough, in particular we have $GTV_{n,\veps_n}(\1_{\Y_n})= E_n(\1_{\Y_n})$ and also since $u \in \ind(D)$, we have $E(u)=TV(u)$. These facts together with the previous chain of inequalities imply the result.
\end{proof}

\subsection{Compactness}
\label{subsection:compactness}
\begin{lemma}[Compactness]  \label{com}
Any subsequence of $\{ u^*_n \}_{n \geq 1}$ or $\{ u^{**}_n \}_{n \geq 1}$ of minimizers of $E_n$
(defined in \eqref{l1-min} and \eqref{l1-min2}) has a further subsequence that converges in the $TL^1$-sense.
\end{lemma}
\begin{proof}
Let $u^{*}_{n},u^{**}_{n}$ denote minimizing sequences. To show that any subsequence of $u^{*}_{n}$ has a convergent subsequence, it suffices to show that both
\begin{align}
\limsup_{n \to \infty} \; GTV_{n,\veps_n}(u^{*}_n) < +\infty \label{eq:tvbound}\\
\limsup_{n \to \infty} \; \|u^{*}_n\|_{L^{1}(\nu_n)} < +\infty \label{eq:l1bound}
\end{align}
hold due to Theorem 1.2 in \cite{GTS}. From the $\Gamma$-convergence established in Proposition \ref{prop:gamma} and from the proof of Proposition \ref{comp_gen} it follows that \eqref{eq:tvbound} is satisfied for both minimizing sequences. Recall that $u^{*}_n = \one_{\Y^*_n}/B_n(\one_{\Y^*_n})$ and that $u^{**}_n = \one_{\Y^{*c}_n}/B_n(\one_{\Y^{*c}_n}),$ where $\Y^*_n$ denotes an optimal balanced cut.

To show \eqref{eq:l1bound}, consider first the balance term that corresponds to the Cheeger Cut. Define a sequence $v_n$ as follows. Set $v_n := u^{*}_n$ if $|\Y^*_n| \leq |\Y^{*c}_n|$ and $v_n = u^{**}_n$ otherwise. 
It then follows that

$$
\|v_n\|_{L^{1}(\nu_n)}  = \frac{ \min\{ |\Y^*_n|,|\Y^{*c}_n|\} }{ \min\{ |\Y^*_n|,|\Y^{*c}_n|\} } =1.
$$
Also, note that $GTV_{n,\veps_n}(v_n) = GTV_{n,\veps_n}(u^{*}_n) $. Thus \eqref{eq:tvbound} and \eqref{eq:l1bound} hold for $v_n,$ so that any subsequence of $v_n$ has a convergent subsequence in the $TL^1$-sense. Let $v_{n_k} \converges{TL^1} v$ denote a convergent subsequence. Now observe that by construction $v_{n_k}$ minimizes $E_{n_k}$ for every $k$. Thus, it follows from Proposition \ref{prop:gamma} and general properties of $\Gamma$-convergence (see Proposition \ref{comp_gen}), that $v$ minimizes $E$ and in particular $v$ is a normalized characteristic function, that is, $v = \one_{A}/B(\one_{A})$ for some $A\subseteq D$ with $B(\one_{A}) \not =0$. Since $B_{n_k}(\one_{\Y^*_{n_k}})= B_{n_k}(\one_{\Y^{*c}_{n_k}})$, $v_{n_k} \converges{TL^1} v$ implies that
$$\frac{1}{B_{n_k}(\Y^*_{n_k})} \rightarrow \frac{1}{B(A)}.$$
Therefore, for large enough $k$ we have

$$  \| u^{*}_{n_k} \|_{L^1(\nu_{n_k})} \leq \frac{1}{B_{n_k}(\Y^*_{n_k})} \leq \frac{2}{B(A)}    $$
and
$$  \| u^{**}_{n_k} \|_{L^1(\nu_{n_k})} \leq \frac{1}{B_{n_k}(\Y^{*c}_{n_k})} = \frac{1}{B_{n_k}(\Y^*_{n_k})} \leq \frac{2}{B(A)}.    $$

We conclude that $\| u^{*}_{n_k} \|_{L^1(\nu_{n_k})}$ and $\| u^{**}_{n_k} \|_{L^1(\nu_{n_k})} $ remain bounded, so that both minimizing subsequences satisfy \eqref{eq:l1bound} and \eqref{eq:tvbound} simultaneously. This yields compactness in the Cheeger Cut case.

Now consider the balance term $B(u) = B_{ {\rm R}}(u)$ that corresponds to the Ratio Cut. Define a sequence $v_{n} := u^{*}_n - \mean_{n}(u^{*}_n),$ and note that $GTV_{n,\veps_n}(v_n) = GTV_{n, \veps_n}(u^{*}_{n})$ since the total variation is invariant with respect to translation. It then follows that
$$
 \|v_n\|_{L^{1}(\nu)} =  \int_{D} |u^{*}_n(x) - \mean_{\rho}(u^{*}_n)| \rho(x) \; \rd x = B(u^{*}_n)=1.
$$
Thus the sequence $\left\{ v_n \right\}_{n \in \N}$ is precompact in $TL^1$. Let $v_{n_k} \converges{TL^1} v$ denote a convergent subsequence. Using a stagnating sequence of transportation maps $\left\{T_{n_k} \right\}_{k \in \N}$ between $\nu$ and the sequence of measures $\left\{ \nu_{n_k} \right\}_{k \in \N}$, we have that $v_{n_k} \circ T_{n_k} \converges{L^1(\nu)} v$. By passing to a further subsequence if necessary, we may assume that $v_{n_k}\circ T_{n_k}(x) \to v(x)$ for almost every $x$ in $D$.

For any such $x,$ we have that either $T_{n_k}(x) \in \Y^*_{n_k}$ or $T_{n_k}(x) \in \Y^{*c}_{n_k}$ so that either
$$
v_{n_k} \circ T_{n_k}(x) = \frac{1}{2|\Y^*_{n_k}|} \quad \text{ or } \quad v_{n_k}\circ T_{n_k}(x) = -\frac{1}{2|\Y^{*c}_{n_k}|}.
$$
Now, by continuity of the balance term, we have
$$  B(v) = \lim_{k \rightarrow \infty} B_{n_k}(v_{n_k}) = 1,   $$
and also
$$ \mean_{\rho}(v) = \lim_{k \rightarrow \infty} \mean_{n_k}(v_{n_k})=0.  $$
In particular the measure of the region in which $v$ is positive is strictly greater than zero, and likewise the measure of the region in which $v$ is negative is strictly greater than zero. It follows that both $|\Y^*_{n_k}|$ and $|\Y^{*c}_{n_k}|$ remain bounded away from zero for all $k$ sufficiently large. As a consequence, the fact that 
$$
 \| u^{*}_{n_k} \|_{L^1(\nu_{n_k})} = \frac{1}{2|\Y^{*c}_{n_k}|}, \qquad  \| u^{**}_{n_k} \|_{L^1(\nu_{n_k})} = \frac{1}{2|\Y^*_{n_k}|}, 
$$
implies that both \eqref{eq:tvbound} and \eqref{eq:l1bound} hold along a subsequence, yielding the desired compactness.
\end{proof}

\subsection{Conclusion of the proof of Theorem \ref{main}}
\label{subsec:ConclusionProof}
 
We may now turn to the final step of the proof. From Proposition \ref{comp_gen}, we know that any limit point of $\{ u^*_n \}_{n \in \N}$ ( in the $TL^1$ sense) must equal $u^*$ or $u^{**}$. As a consequence, for any subsequence ${u^*_{n_k}}$ that converges to $u^*$ we have that $\one_{ \Y_{n_k}^*} \converges{TL^1} \one_{A^*} $ by lemma \ref{continuity}, while $\one_{\Y_{n_k}^*} \converges{TL^1} \one_{{A^*}^c} $ if the subsequence converges to $u^{**}$ instead. Moreover, in the first case we would also have $\one_{ \Y_{n_k}^{*c}} \converges{TL^1} \one_{A^{*c}} $ and in the second case $\one_{\Y_{n_k}^{c}} \converges{TL^1} \one_{{A^*}} $. Thus in either case we have 
$$ \left\{ \Y_{n_k}^*, \Y_{n_k}^{*c} \right\}  \converges{TL^1}\left\{ A^*, {A^*}^c \right\} $$

Thus, for any subsequence of $\left\{ \Y_n^*, {\Y_n^{*c}} \right\}_{n \in \N}$ it is possible to obtain a  further subsequence converging to $\{A^*,{A^*}^c\}$, and thus the full sequence converges to $\{A^*,{A^*}^c\}$.

\section{Consistency of multiway balanced cuts}

Here we prove  Theorem \ref{main2}.

\label{sec:TheoremMulti-Class}
Just as what we did in the two-class case, the first step in the proof of Theorem \ref{main2} involves a reformulation of both the balanced graph-cut problem \eqref{eq:multi1:Body} and the analogous balanced domain-cut problem \eqref{eq:multi1} as equivalent minimizations defined over spaces of functions and not just spaces of partitions or sets.

We let $B_n(u_n):= \mean_n(u_n)$ for $u_n \in L^1(\nu_n)$ and $
B(u) := \mean_\rho(u)$ for $u \in L^1(\nu)$, to be the corresponding  balance terms. Given this balance terms, we let $\ind_n(D)$ and $\ind(D)$ be defined as in \eqref{ind_n} and \eqref{ind} respectively.

We can then let the sets $\Part_n(D)$ and $\Part(D)$ to consist of those collections $\mathcal{U} = (u_1,\ldots,u_{R})$ comprised of exactly $R$ disjoint, normalized indicator functions that cover $D$. The sets $\Part_n(D)$ and $\Part(D)$ are the multi-class analogues of $\ind_n(D)$ and $\ind(D)$ respectively. Specifically, we let
\begin{align}
\Part_n(D) &=  \left\{ (u_1^n,\ldots,u_R^n)\: :   u_r^n \in \ind_n(D) ,\;\int_{D} u_r^n(x) u_s^n(x) \; \rd \nu_n(x) = 0 \;\; \text{if} \;\; r \neq s \text{, } \sum_{r=1}^{R}u_r^n > 0   \right\}\\
\Part(D) &= \left\{ (u_1,\ldots,u_R)\: :   u_r \in \ind(D),\;\int_{D} u_r(x) u_s(x) \; \rd \nu(x) = 0 \;\; \text{if} \;\; r \neq s \text{, } \sum_{r=1}^{R}u_r > 0  \right\}.
\end{align}
Note for example that if  $\mathcal{U} = (u_1,\ldots,u_{R}) \in \Part(D)$, then the functions $u_r$
are normalized indicator functions, 
 $u_r = \one_{A_r}/|A_r|$ for $1\leq r \leq R$, and the orthogonality constraints 
imply that $\{A_1,\ldots,A_R\}$ is a collection of pairwise disjoint sets (up to Lebesgue-null sets). Additionally, the condition that $\sum_{r=1}^{R} u_r > 0$ holds almost everywhere implies that the sets $\{A_1,\ldots,A_R\}$ cover $D$ up to Lebesgue-null sets. 

With these definitions in hand, we may follow the same argument in the two-class case to conclude that that the minimization
\begin{equation}  \label{eq:multi_final}
\text{Minimize \;\;\;  }  E_n(\mathcal{U}_n) := \begin{cases} 
\sum^{R}_{r=1} \; GTV_{n,\veps_n}(u_r^n) & \text{ if } \mathcal{U}_n  \in  \Part_{n}(D) \\
+ \infty & \text{ otherwise }
\end{cases} \\
\end{equation}
is equivalent to the balanced graph-cut problem \eqref{eq:multi1:Body}, while the minimization
\begin{equation}  \label{eq:multi_final_domain}
\text{Minimize \;\;\;  }  E(\mathcal{U}) := \begin{cases} 
\sum^{R}_{r=1} \; TV(u_r) & \text{ if } \mathcal{U}  \in  \Part(D) \\
+ \infty & \text{ otherwise }
\end{cases} \\
\end{equation}
is equivalent to the balance domain-cut problem \eqref{eq:multi1}. 

At this stage, the proof of Theorem \ref{main2}, is completed by following the same steps as in the two-class case. In particular we want to show that $E_n$ defined in \eqref{eq:multi_final} $\Gamma$-converges in the $TL^1$-sense to $\sigma_\eta E$, where $E$ is defined in \ref{eq:multi_final_domain}. That is, we want to prove the following.
\begin{proposition}{\rm ($\Gamma$-Convergence)}\label{prop:gamma2}
Let domain $D$, measure $\nu$, kernel $\eta$, sequence $\{\veps_n\}_{n \in \N}$,
sample points $\{\x_i\}_{i \in N}$, and graph $\G_n$
 satisfy the assumptions of Theorem \ref{main}.
 Consider  functionals $E_n$ of \eqref{eq:multi_final} and $E$ of \eqref{eq:multi_final_domain}. Then
 \[ E_n \overset{\Gamma}{\longrightarrow} \sigma_n E \quad \te{ with respect to } (TL^1)^{\otimes R} \te{ metric as }  n \to \infty. \]
 That is
\begin{enumerate}
\item For {  any} $\mathcal{U} \in [L^1(D)]^R$ and {  any} sequence $\mathcal{U}_{n} \in (L^1(\nu_n))^R$ that converges to $\mathcal{U}$ in the $TL^1$ sense, 
\begin{equation}\label{eq:liminf2}
E(\mathcal{U}) \leq \liminf_{n \to \infty} \; E_{n}(\mathcal{U}_n).
\end{equation} 
\item For {  any} $\mathcal{U} \in [L^1(D)]^R$ there exists {  at least one} sequence $\mathcal{U}_{n}$ that both converges to $\mathcal{U}$ in the $TL^1$-sense and also satisfies
\begin{equation}\label{eq:limsup2}
\limsup_{n \to \infty} \; E_{n}(\mathcal{U}_n) \leq E(\mathcal{U}).
\end{equation} 
\end{enumerate}
\end{proposition}

\begin{remark}
We remark that all the types of convergence for vector-valued functions are to be understood as component-wise convergence in the corresponding topology. This helps us clarify the way the $TL^1$ -convergence is considered in Proposition \ref{prop:gamma2}.  
\end{remark}

Assuming Proposition \ref{prop:gamma2}, the following lemma follows in a straightforward way. We omit its proof since it follows analogous arguments to the ones used in the proof of Proposition \ref{com}.  
\begin{lemma}[Compactness]\label{lem:compactness2}
Any subsequence of $\{ \mathcal{U}^*_n \}_{n \geq 1}$ of minimizers to \eqref{eq:multi_final} has a further subsequence that converges in the $TL^1$-sense.
\end{lemma}

Finally, due to Proposition \ref{prop:gamma2} and Lemma \ref{lem:compactness2}, the arguments presented in Subsections  \ref{OutlineTwo-Class} and \ref{subsec:ConclusionProof} can be adapted in a straightforward way to  complete the proof of Theorem \ref{main2}. So we focus on the proof of Proposition \ref{prop:gamma2}, where arguments not present in the two-class case are needed. On one hand, this is due to the presence of the orthogonality  constraints in the definition of $\mathcal{M}_n(D)$ and $\mathcal{M}(D)$, and on the other hand,  from a geometric measure theory perspective, due to the fact that an arbitrary partition of the domain $D$ into more than two sets can not be approximated by smooth partitions as multiple junctions appear when more than two sets in the partition meet.

\subsection{Proof of Proposition \ref{prop:gamma2}}  
The next lemma is the multiclass analogue of Lemmas \ref{Bcont} and \ref{continuity} combined.
\begin{lemma} \label{lem:continuity2} {\rm (i)} If $\mathcal{U}_k \to \mathcal{U}$ in $(L^1(\nu))^R$ then  $B(u^k_r) \to B(u_r)$ for all $1 \leq r \leq R$. {\rm (ii)} The set $\Part(D)$ is closed in $L^1(\nu)$. {\rm(iii)} If $\left\{\mathcal{U}_n \right\}$ is a sequence with $\mathcal{U}_n  \in (L^1(\nu_n))^R$ and $\mathcal{U}_n \converges{TL^1} \mathcal{U}$ for some $\mathcal{U} \in (L^1(\nu))^R$, then $B_n(u^n_r) \to B(u_r)$ for all $1 \leq r \leq R$. {\rm (iv)} If $u_n = \tilde{\1}_{\Y_n}$, where $\Y_n \subset \X_n$, converges to $u = \tilde{\1}_{A}$ in the $TL^1$-sense, then $\one_{\Y_n} $ converges to $\one_{A}$ in the $TL^1$-sense.
 \end{lemma}
\begin{proof}
Statements {\rm(i)}, {\rm(iii)} and {\rm(iv)} follow directly from the proof of Proposition \ref{continuity}. 

In order to prove the second statement, suppose that a sequence $\left\{ \mathcal{U}_k \right\}_{k \in \N}$ in $\mathcal{M}(D)$ converges to some $\mathcal{U}$ in $(L^1(\nu))^R$. We need to show that $\mathcal{U} \in \mathcal{M}(D)$.  First of all note that for every $1\leq r \leq R$, $u_{r}^k \converges{L^1(\nu)} u_r$. Since $u_r^k \in \ind(D)$ for every $k \in \N$, and since $\ind(D)$ is a closed subset of $L^1(\nu)$ (by Proposition \ref{continuity}), we deduce that $u_r \in \ind(D)$ for every $r$.

The orthogonality condition follows from Fatou's lemma. In fact, working along a subsequence we can without the loss of generality assume that for every $r$, $u_r^k \rightarrow u_r $ for almost every $x$ in $D$. Hence, for $r \not = s$ we have
$$ 0 \leq \int_{D}u_r(x) u_s(x) d \nu(x) = \int_{D} \liminf_{k \rightarrow \infty} ( u_r^k(x)u_s^k(x) ) \rd \nu(x)  \leq \liminf_{k \rightarrow \infty} \int_{D} u_r^k(x)u_s^k(x) d\nu(x)=0   $$

Now let us write $u_r^k = \one_{A_r^k} / B(\one_{A_r^k})$ and $u_r = \one_{A_r^k} / B(\one_{A_r})$. As in the proof of Proposition \eqref{continuity} we must have $B(\one_{A_r^k}) \rightarrow B(\one_{A_r})$ as $k \rightarrow \infty$. Thus, for almost every $x \in D$ 
$$   \sum_{r=1}^{R}u_r(x)  = \lim_{k \rightarrow \infty}   \sum_{r=1}^{R}u_r^k(x) \geq \lim_{k \rightarrow \infty } \min_{r=1, \dots, R} \frac{1}{B(\one_{A_r^k})}   = \min_{r= 1, \dots, R} \frac{1}{B(\one_{A_r})} >0.   $$
\end{proof}

\begin{proof}[Proof of Proposition \ref{prop:gamma2}]

\textbf{Liminf inequality.} The proof of \eqref{eq:liminf2} follows the approach used in the two-class case. Let $\mathcal{U}_n \converges{TL^1} \mathcal{U}$ denote an arbitrary convergent sequence. As $\Part(D)$ is closed, if $\mathcal{U} \notin \Part(D)$ then as in the two-class case, it is easy to see that $\mathcal{U}_n \notin\Part_n(D)$ for all $n$ sufficiently large. The inequality \eqref{eq:liminf2} is then trivial in this case, as both sides of it are equal to infinity. Conversely, if $\mathcal{U} \in \Part(D)$ then we may assume that  $\mathcal{U}_{n} \in \Part_n(D)$ for all $n,$ since only those terms with $\mathcal{U}_{n} \in \Part_n(D)$ can contribute non-trivially to the limit inferior. In this case we easily have
\begin{align*}
\liminf_{n \to \infty} \; E_n(\mathcal{U}_n)  = \liminf_{n \to \infty} \; \sum^{R}_{r=1} \;\; GTV_{n, \veps_n}\left(u_r^n \right) & \geq \sum^{R}_{r=1} \;\; \liminf_{n \to \infty} \; GTV_{n, \veps_n}(u^n_r) \\
&\geq \sigma_\eta \sum^{R}_{r=1} \;\; TV(u_r) = \sigma_\eta E(\mathcal{U}).
\end{align*}
The last inequality follows from Theorem \ref{DiscreteGamma}. This establishes the first statement in Proposition \ref{prop:gamma2}. 
\nc

\textbf{Limsup inequality.} We now turn to the proof of \eqref{eq:limsup2}, which is significantly more involved than the two-class argument due to the presence of the orthogonality constraints.  
Borrowing terminology from the $\Gamma$-convergence literature, we say that $\mathcal{U} \in (L^1(\nu))^R$ has a \emph{recovery sequence} when there exists a sequence $\mathcal U_n \in (L^1(\nu_n))^R$ such that
\eqref{eq:limsup2} holds.
 To show that each $\mathcal{U} \in (L^{1}(\nu))^R$ has a recovery sequence, we first remark that due to general properties of the $\Gamma$-convergence, it is enough to verify \eqref{eq:limsup2} for $\mathcal{U}$ belonging to a dense subset of $\mathcal{M}(D)$ with respect to the energy $E$ (see Remark \ref{DenseGamma}). 
 We furthermore remark that it is enough to consider 
 $\mathcal{U} = (u_1,\ldots,u_R) \in (L^{1}(D))^R$ for which $E(\mathcal{U}) <\infty$, as the other case is trivial.
So we can consider  $\mathcal{U} \in \Part(D)$ that satisfy 
$$
\sum^{R}_{r=1} \; TV(u_{r}) < \infty.
$$ 
Let $u_{r} = \one_{A_r}/B(\one_{A_r})$ and let $c_{0} := \max \{ B(\one_{A_1}), \ldots , B(\one_{A_R}) \} $ denote the size of the largest set in the collection. The fact that $E(\mathcal{U}) < \infty$ then implies
$$
TV(\one_{A_r}) \leq c_{0} \; TV(u_r) \leq c_{0} \; \sum^{R}_{r=1} \; TV(u_{r}) < \infty,
$$
so that all sets $\{A_1,\ldots,A_{R}\}$ in the collection defining $\mathcal{U}$ have finite perimeter. Additionally because $\mathcal{U} \in \Part(D)$ implies that any two sets $A_r,A_s$ with $r \neq s$ have empty intersection up to a Lebesgue-null set, we may freely assume without the loss of generality that the sets $\left\{ A_1,\ldots,A_{R}\right\}$ are mutually disjoint.

Let us define sets with \emph{piecewise (PW) smooth} boundary to be the subsets of $\R^d$ whose boundary is a subset of finitely many open $d-1$-dimensional manifolds embedded in $\R^d$. 
We  start by constructing a recovery sequence for a $\mathcal{U}$ whose defining sets $\{A_1,\ldots,A_{R}\}$ are of the form $A_r= B_r\cap D$, where $B_r$ has piecewise smooth boundary and satisfies $|D \one_{B_r}|_{\rho^2}(D)=0$. We say that such $\mathcal{U}$ is \textit{induced by piecewise smooth sets}. We later prove that such partitions are dense among partitions of $D$ by sets of finite perimeters. \footnote{Note that unlike in the two-class case, due to "triple junctions", one cannot approximate a general partition by a partition with sets with smooth boundaries. This makes the construction more complicated.}

\textbf{Contructing a recovery sequence for $\mathcal{U}$ induced by piecewise smooth sets.} Let $\Y^n_r = A_r \cap \X_n$ denote the restriction of $A_r$ to the first $n$ data points. Now, let us consider the transportation maps $ \left\{ T_{n} \right\}_{n \in \N}$ from Proposition \ref{thm:InifinityTransportEstimate}.  We let $A_n^r$ be the set for which $\one_{A^n_r} = \one_{\Y_n^r} \circ T_n $. 

We first notice that the fact that $B_r$ has a piecewise smooth boundary in $\R^{d}$ and the fact that $||Id - T_n||_{\infty} \rightarrow 0$, imply that
\begin{align}
\| \one_{A^n_r} - \one_{A_r}\|_{L^{1}(\nu)} \leq C_{0}(B_r) \; ||Id - T_n||_{\infty},
\label{aux:EstimateLimsupMulti}
\end{align}
where $C_0(B_r)$ denotes some constant that depends on the set $B_r$. This inequality follows from the formulas for the volume of tubular neighborhoods (see \cite{Weyl}). In particular, note that by the change of variables \eqref{ChangeOfVariables} we have, $|\Y_r^n|= |A^n_r| \rightarrow |A_r|$ as $n \rightarrow \infty$, so that in particular we can assume that $|\Y_r^n| \not =0$.  We define $u^n_r := \one_{\Y^n_r}/|\Y^n_r|$ as the corresponding normalized indicator function. We claim that $\mathcal{U}_n := (u^{n}_1,\ldots,u^{n}_{R})$ furnishes the desired recovery sequence. 

To see that $\mathcal{U}_n \in \Part_n(D)$ we first note that each $u^{n}_r \in \ind_n(D)$ by construction. On the other hand, the fact that  $\{A_1,\ldots,A_{R}\}$ forms a partition of $D$ implies that $\{\Y^n_1,\ldots,\Y^n_R \}$ defines a partition of $\X_n$. As a consequence,
$$
E_{n}(\mathcal{U}_n) = \sum^{R}_{r=1} \; GTV_{n, \veps_n}( u^n_r) 
$$
by definition of the $E_n$ functionals. 


Using \eqref{aux:EstimateLimsupMulti}, we can proceed as in remark 5.1 in \cite{GTS}. In particular, we can assume that $\eta$ has the form $\eta(z) = a$ for $z < b$ and $\eta(z) = 0$ otherwise. We set  $\tilde{\veps}_n := \veps_n + \frac{2}{b}|| Id - T_n||_{\infty}$. Recall that by assumption  $|| Id - T_n||_{\infty} \ll \veps_n$, and thus $\tilde \veps_n$ is a perturbation of $\veps_n$.
Define the non-local total variation $\tTV_{\tilde{\veps}_n}$ of an integrable function $u \in L^1(\nu)$ as
\begin{align*}
\tTV_{\tilde{\veps}_n}(u) := \frac{1}{\tilde{\veps}^{d+1}_n} \int_{\D \times \D} \eta\left(\frac{|x-y|}{\tilde{\veps}_n}\right)|u(x)-u(y)| \rho(x) \rho(y)\; \rd x \rd y.
\end{align*} 
Using the definition of $\tilde{\veps}_n$, and the form of the kernel $\eta$, we deduce that for all $n\in \N$, and almost every $x,y \in D$ we have
$$ \eta\left( \frac{|T_n(x)-T_n(y)|}{\veps_n}   \right) \leq \eta \left(\frac{|x-y|}{\tilde{\veps}_n}\right). $$
This inequality an a change of variables (see \ref{ChangeOfVariables}) implies that
\begin{equation*}
\frac{\veps_n^{d+1}}{\tilde{\veps}_n^{d+1}}GTV_{n, \veps_n}( \one_{\Y^n_r})    \leq \tTV_{\tilde{\veps}_n}(\one_{A^n_r} ).
\end{equation*}

A straightforward computation shows that there exists a constant $K_0$ such that
\begin{align*}\label{eq:sich}
| \tTV_{\tilde{\veps}_n}( \one_{A^n_r}) - \tTV_{\tilde{\veps}_n}( \one_{A_r}) | \leq \frac{ K_{0} }{\tilde{\veps}_n} \| \one_{A^n_r} - \one_{A_r}\|_{L^{1}(\nu)} \leq K_0 C_0(B_r) \frac{||Id - T_n ||_\infty}{\tilde{\veps}_n}.
\end{align*}
Since $\frac{\veps_n}{\tilde{\veps}_n} \rightarrow 1$, the previous inequalities imply that
$$  \limsup_{n \in \N} GTV_{n \veps_n}(\one_{\Y^n_r}) \leq  \limsup_{n \in \N}  \tTV_{\tilde{\veps}_n}( \one_{A^n_r}) = \limsup_{n \in \N} \tTV_{\tilde{\veps}_n}( \one_{A_r}).  $$

Finally, from remark 4.3 in  \cite{GTS} we deduce that
\begin{align*}
\limsup_{n \to \infty} \; \tTV_{\tilde{\veps}_n}( \one_{A_r^n}) \leq \sigma_\eta TV(\one_{A_r}),
\end{align*}
and thus we conclude that  $\limsup_{n \rightarrow \infty} \; GTV_{n,\veps_n}( \one_{A^r}) \leq \sigma_\eta TV( \one_{A_r})$. As a consequence we have
$$
\limsup_{n \to \infty} \; GTV_{n,\veps_n}(u^{n}_r) = \limsup_{n \to \infty} \; \frac{ GTV_{n,\veps_n}( \one_{\Y^n_r}) }{ B_n(\one_{\Y^n_r}) } \leq  \sigma_\eta \frac{ TV( \one_{A_r}) }{ B(\one_{A_r}) }
$$
for each $r$, by continuity of the balance term. From the previous computations we conclude that $E_n(\mathcal{U}_n) \to E(\mathcal{U})$, and from \ref{aux:EstimateLimsupMulti}, we deduce that $\mathcal{U}_n \to \mathcal{U}$ in the $TL^1$-sense, so that $\mathcal{U}_{n}$ does furnish the desired recovery sequence.

\textbf{Density.} To establish Proposition \ref{prop:gamma2}, we show that for any given $ \mathcal{U}=(\tilde{\one}_{A_1},\ldots,\tilde{\one}_{A_{R}})$ where each of the sets $A_r$ has finite perimeter, there exists a sequence   $\left\{  \mathcal{U}_m=(\tilde{\one}_{A_1^m}, \dots, \tilde{\one}_{A_R^m}) \right\}_{m \in \N}$, where each of the $\mathcal{U}_m$ is induced by piecewise smooth sets, and such that  for every $r \in \left\{1, \dots,R \right\}$
$$\one_{A_r^m } \converges{L^1(\nu)} \one_{A_r}, $$
and 
$$  \lim_{m \rightarrow \infty}  TV(\one_{A_r^m};D) = TV(\one_{A_r};D). $$

Note that in fact, by establishing the existence of such approximating sequence, it immediately follows that $\mathcal{U}_m \rightarrow \mathcal{U}$ in $(L^1(\nu))^R$ and that $\lim_{m \rightarrow \infty} E(\mathcal{U}_m)  = E(\mathcal{U})$ ( by continuity of the balance terms).  We provide the construction of the approximating sequence $\left\{ \mathcal{U}_m\right\}_{m \in \N}$  through the sequence of three lemmas presented below.
\begin{lemma}
Let $\{A_1, \dots, A_R\}$ denote a collection of open and bounded sets with smooth boundary in $\R^d$ that satisfy
\begin{equation}
 \mathcal{H}^{d-1}( \partial A_r \cap \partial A_s ) = 0 \:, \: \forall r \not =s.
 \label{HypothesisBoundaries}
 \end{equation}   
Let $D$ denote an open and bounded set. Then there exists a permutation $\pi: \left\{1, \dots, R \right\} \rightarrow \left\{ 1, \dots, R \right\}$ such that
$$  TV(\one_{A_{\pi(r)} \setminus \bigcup_{s=r+1}^{R} A_{\pi(s)}} ; D ) \leq TV(\one_{A_{\pi(r)}}; D), \: \forall r \in \left\{ 1, \dots, R \right\} .  $$
\label{LemmaPermutation}
\end{lemma}
\begin{proof}
The proof is by induction on $R$. \textbf{Base case:} Note that if $R=1$ there is nothing to prove. \textbf{Inductive Step:} Suppose that the result holds when considering any $R-1$ sets as described in the statement. Let $A_1, \dots, A_R$ be a collection of sets as in the statement. By the induction hypothesis it is enough to show that we can find $r \in \left\{ 1, \dots, R \right\}$ such that 
         \begin{equation}
TV(\one_{A_{r} \setminus \bigcup_{s \not = r } A_{s}} ;D ) \leq TV(\one_{A_{r}};D).  
\label{IneqLemma1}
\end{equation} 
To simplify notation, denote by $\Gamma_r$ the set $\partial A_r$ and define $a_{rs}$ as the quantity
$$a_{rs}:=   \int_{\Gamma_r \cap ( A_s \setminus \bigcup_{k\not =r, k \not = s} A_k )\cap D} \rho^2(x)  \;\rd \mathcal{H}^{d-1}(x). $$
%
%
%
Hypothesis \eqref{HypothesisBoundaries} and \eqref{TVsmoothSets} imply that the equality 
\begin{equation}
TV( \one_{A_r \setminus \bigcup _{s \not =r} A_s} ; D ) = \int_{\Gamma_r \cap ( \bigcup_{k \not =r} A_k   )^c \cap D} \rho^2(x) \;\rd\mathcal{H}^{d-1}(x) + \sum_{s :\:s \not = r} a_{sr}
\label{SecondIneqLemma1}  
\end{equation}
holds for every $r \in \left\{1, \dots, R \right\},$ as does the inequality
\begin{equation}
TV(\one_{A_r}; D ) \geq  \int_{\Gamma_r \cap ( \bigcup_{k \not =r} A_k   )^c \cap D} \rho^2(x) \;\rd\mathcal{H}^{d-1} + \sum_{s :\: s \not =r}a_{rs}.
\label{FirstIneqLemma1}
\end{equation}
If $  TV(\one_{A_{r} \setminus \bigcup_{s \not = r } A_{s}} ; D ) > TV(\one_{A_{r}}; D)$ for every $r$ then \eqref{FirstIneqLemma1} and \eqref{SecondIneqLemma1} would imply that
$$ \sum_{s :\:s \not = r} a_{sr} > \sum_{ s  :\:s \not =r} a_{rs}, \: \: \forall r,     $$   
which after summing over $r$ would imply 
$$  \sum_{r=1}^{R} \sum_{s :\:s \not = r}   a_{sr}  > \sum_{r=1}^{R} \sum_{s :\:s \not = r}   a_{rs} = \sum_{r=1}^{R} \sum_{s :\:s \not = r}   a_{sr}.  $$
This would be a contradiction. Hence there exists at least one $r$ for which \eqref{IneqLemma1} holds.
\end{proof}

\begin{lemma}
Let $D$ denote an open, bounded domain in $\Rd$ with Lipschitz boundary and let $(B_{1},\ldots,B_{R})$ denote a collection of $R$ bounded and mutually disjoint subsets of $\R^d$ that satisfy
\begin{align*}
(\mathrm{i}) \;\; TV(\one_{B_r};\Rd) < +\infty \quad \text{,} \quad
(\mathrm{ii}) \;\; |D \one_{B_r}|_{\rho^2}(\partial D) = 0 \quad \text{and} \quad
(\mathrm{iii}) \;\; D \subseteq \cup_{r=1}^{R} B_r .
\end{align*}
Then there exists a sequence of mutually disjoint sets $\{A^{m}_1,\ldots,A^{m}_{R}\}$ with piecewise smooth boundaries which cover $D$ and satisfy
\begin{equation}\label{eq:result1}
\one_{A^{m}_{r}} \to_{L^{1}(\Rd)} \one_{B_r} \quad \text{and} \quad \lim_{m \to \infty} \; TV(\one_{A^m_r};D) = TV(\one_{B_r};D)
\end{equation}
for all $1\leq r \leq R$.
\label{Lemma8}
\end{lemma}
\begin{proof}
First of all note that $TV(\one_{B_r};\Rd)$ and $|D \one_{B_r}|_{\rho^2} (\partial D)$ are defined considering $\rho$ as a function from $\R^d$ into $\R$. We are using the extension considered when we introduced the weighted total variation at the beginning of subsection \ref{sec:TV}. Given that $\rho^2: \R^d \rightarrow (0,\infty)$ is lower semi-continuous and bounded below and above by positive constants then, it belongs to the class of weights considered in \cite{Baldi} where the weighted total variation is studied. 

Let $\{\gamma_k\}_{k \in \N}$ denote some sequence of positive reals converging to zero and
\begin{equation*}
J_{k}(x) := \frac{1}{\gamma^{d}_{k}}J\left( \frac{|x|}{\gamma_k} \right)\;  J \geq 0, \; J \in C^{\infty}_{c}( [0,1] ), \; \; \int_{\R^{d}} J(x) \; \rd x = 1
\end{equation*}
a corresponding sequence of positive, radially symmetric mollifiers. Let $D_{k} := \{ x \in \Rd : \mathrm{dist}(x,D) < \gamma_k \}$ denote the open $\gamma_k$-neighborhood of the domain $D$. For each $k \in \mathbb{N}$ and each $B_{r}$ in the collection let
\begin{equation*}
u^{k}_{r} := J_{k} * \textbf{1}_{B_r}
\end{equation*}
denote a smoothed version of the characteristic function.

For any test function $\Phi \in C^{1}_{c}(D: \R^d)$ with $|\Phi(x)| \leq \rho^2(x),$ we have 
\begin{equation*}
\int_{D} u^{k}_{r} \; \mathrm{div}( \Phi(x) ) \; \rd x =- \int_{D_{k} } \one_{B_r} \; \mathrm{div} (J_k * \Phi(y))  \; \rd y \leq |D\textbf{1}_{B_r}|_{\rho^2}(D_{k}).
\end{equation*}
The equality follows from the symmetry of $J_k$ and the fact that $J_k * \Phi$ has support within $D_k$ while the inequality follows from the fact that $|J_k * \Phi| \leq \rho^2$ so it produces an admissible test function in the definition of the total variation. As a consequence,
\begin{equation*}
\limsup_{ k \to \infty} \; TV(u^{k}_r;D) \leq \limsup_{ k \to \infty} \; |D\one_{B_r}|_{\rho^2}(D_{k}) =  |D\one_{B_r}|_{\rho^2}(\overline{D}) = 
|D\one_{B_r}|_{\rho^2}(D)
\end{equation*}
due to the second assumption in the statement of the lemma. The fact that $u^{k}_{r} \to_{L^{1}(\Rd)} \one_{B_r}$ combines with the lower-semicontinuity of the total variation to imply
\begin{equation*}
TV(\one_{B_r};D) \leq \liminf_{k \to \infty} \; TV(u^{k}_r ; D) \leq \limsup_{k \to \infty} \; TV(u^{k}_r;D) \leq TV(\one_{B_r};D).
\end{equation*}
In other words, these sequences satisfy

\begin{equation*}
u^{k}_{r} \to_{L^1(\Rd)} \one_{B_r}, \qquad TV(u^{k}_r ; D) \to TV(\one_{B_r};D),  \qquad 0 \leq u^{k}_{r}(x) \leq 1 \;\; \forall x \in \R^{d}.
\end{equation*}

The $(u^k_1,\ldots,u^{k}_R)$ also satisfy one additional property that will prove useful: there exists a constant $\alpha > 0$ so that
$$
\Sigma^{k}(x) := \sum^{R}_{r=1} u^{k}_{r}(x) \geq \alpha > 0 \quad \text{for all} \quad x \in D.
$$
To see this, note that the fact that $D$ is an open and bounded set with Lipschitz boundary implies that there exists a cone $C \subseteq \R^d$ with non-empty interior, a family of rotations $\left\{ R_{x} \right\}_{x \in D}$ and $\zeta>0$ such that for every $x \in D,$
$$  x + R_x(C \cap B(0,\zeta))  \subseteq D. $$

The fact that $J$ is radially symmetric then implies that for every $x \in D$,

\begin{align*}
\int_{D} J_k(x-y)dy  \geq \int_{x + R_x(C \cap B(0,\zeta)) } J_k(x-y) dy = \\
\int_{C \cap B(0,\zeta)} J_k(y)dy = \int_{C \cap B(0,\frac{\zeta}{\gamma_k})} J(y)dy  \geq \alpha  >0
\end{align*}

for some positive constant $\alpha$. The summation $\Sigma^{k}(x)$ of all $u^{k}_{r}$ therefore satisfies the pointwise estimate
\begin{equation*}
\Sigma^{k}(x) := \sum^{R}_{r=1} u^{k}_{r}(x) = \int_{\Rd} J_k(x-y) \sum^{R}_{r=1} \one_{B_r}(y) \; \rd y \geq  \int_{D} J_k(x-y)dy \geq \alpha
\end{equation*}
for all $x \in D$ as claimed.

\textbf{Step 1:}
Now, for each $u^{k}_{r}$ and each $t \in (0,1)$ consider the superlevel set
\begin{equation*}
B^{k}_{r}(t) := \left\{ u^{k}_{r} > t \right\}.
\end{equation*}
The first claim is that, for any fixed $t$ in $(0,1),$ the characteristic function $\one_{ B^{k}_{r}(t)}$ converges in $L^{1}(\R^{d})$ to the characteristic function of the original set. To see this, note that
\begin{equation*}
B^{k}_{r}(t) \setminus B_{r} \subset \left\{ |u^{k}_r - \one_{B_r}| \geq t \right\}.
\end{equation*}
By Chebyshev's/Markov's inequality, if $\mathcal{L}_{d}$ denotes Lebesgue measure in $\Rd$ then
\begin{equation*}
\mathcal{L}_{d}( B^{k}_{r}(t) \setminus B_{r} ) \leq \mathcal{L}_{d} \left( \left\{ |u^{k}_r - \one_{B_r}| > t \right\} \right) \leq \frac{1}{t} \| u^{k}_{r} - \one_{B_r} \|_{L^{1}} \to 0.
\end{equation*}
In a similar fashion,
\begin{equation*}
\mathcal{L}_{d}( B_r \setminus B^{k}_{r}(t) ) \leq \mathcal{L}_{d}\left( \left\{ |u^{k}_r - \one_{B_r}| \geq (1-t) \right\} \right) \leq \frac{1}{1-t} \| u^{k}_{r} - \one_{B_r} \|_{L^{1}} \to 0.
\end{equation*}
As a consequence, it follows that
\begin{equation*}
\int_{\Rd} |\one_{B^{k}_{r}(t)} - \one_{B_r}| \; \rd x = \mathcal{L}_{d}( B^{k}_{r}(t) \setminus B_{r} ) + \mathcal{L}_{d}( B \setminus B^{k}_{r}(t) ) \to 0
\end{equation*}
as claimed.

\textbf{Step 2:}
The next claim is that there exists a set $\mathcal{T} \subset (0,1)$ of full Lebesgue measure with the following property: if $t \in \mathcal{T}$ then $B^{k}_{r}(t)$ has a smooth boundary for all $k$ and all sets $(B^{k}_1(t),\ldots,B^{k}_{R}(t))$ in the collection. To see this, note Sard's lemma (see for example \cite{Leoni}) implies that for any fixed $k \in \mathbb{N}$ the set $B^{k}_{r}(t)$ has smooth boundary up to an exceptional set $\mathcal{T}_{k,r} \subset (0,1)$ of Lebesgue measure zero. Now define the set $\mathcal{T}$ as
\begin{equation*}
\mathcal{T} = (0,1) \setminus \bigcup^{\infty}_{k=1} \bigcup^{R}_{r=1} \; \mathcal{T}_{k,r}.
\end{equation*}
Note that $\mathcal{T}$ has full measure since a countable union of Lebesgue-null sets has measure zero. If $t \in \mathcal{T}$ then it does not lie in any of the exceptional sets, meaning that for each $k$ and each $r$ the set $B^{k}_{r}(t)$ has a smooth boundary.

\textbf{Step 3:}
We use a diagonal argument to construct an approximating sequence of partitions that are not necessarily disjoint, but satisfy the hypotheses of Lemma \eqref{HypothesisBoundaries}. 

For the set $B_1$, Step 1 and lower semi-continuity of the total variation imply that for all $t \in (0,1)$ 
\begin{equation*}
TV(\one_{B_1};D) \leq \liminf_{k \to \infty} \;TV(\one_{B^{k}_1(t)};D).
\end{equation*}
On the other hand, Fatou's lemma combines with the co-area formula to imply
\begin{align*}
\int^{1}_{0} \liminf_{k \to \infty} \;TV(\one_{B^{k}_1(t)};D) \; \rd t \leq \lim_{k \to \infty} \int^{1}_{0} TV(\one_{B^{k}_1(t)};D) \; \rd t = \lim_{k \to \infty} \; TV(u^{k}_{1};D) = TV(\one_{B_1};D).
\end{align*}
In other words,
\begin{equation*}
TV(\one_{B_1}; D) \leq \liminf_{k \to \infty} \; TV(\one_{B^{k}_1(t)};D) \qquad \text{and} \qquad \int^{1}_{0} \liminf_{k \to \infty} \; TV(\one_{B^{k}_1(t)};D) \; \rd t = TV(\one_{B_1};D),
\end{equation*}
which implies
\begin{equation*}
\liminf_{k \to \infty} \;TV(\one_{B^{k}_1(t)};D) = TV(\one_{B_1};D)
\end{equation*}
almost everywhere. In particular, there exists a $t_{1} \in \mathcal{T}$ with $0 < t_{1} < \alpha/R$ and a subsequence $\{k_m\}_{m \in \N}$ with the property that
\begin{equation}\label{eq:nice_boundary}
\partial B^{k_m}_{1}(t_1) \; \text{is smooth} \; \forall  m, \quad  \lim_{m \to \infty} \; TV(\one_{B^{k_{m}}_1(t_1)};D) =  TV(\one_{B_1};D),  \quad \one_{B^{k_{m}}_1(t_1)} \converges{L^{1}(\nu)} \one_{B_1}.
\end{equation}

We now pass to the set $B_2$. As $\partial B^{k_m}_{1}(t_1)$ is smooth and bounded for all $m$, it has zero Lebesgue measure for all $m$ in particular. As $u^{k_m}_2$ is smooth, lemma 2.95 in \cite{Fusco} implies that
$$
\mathcal{H}^{d-1}\left( \partial B^{k_m}_{1}(t_1) \cap \partial B^{k_m}_{2}(t) \right) = 0
$$
for almost every $t \in (0,1)$. Let $\mathcal{T}_{2,m}$ denote the $m^{ {\rm th}}$ exceptional set for which this property does not hold. Define the set
$$
\mathcal{T}_{2} := \mathcal{T} \setminus \bigcup^{\infty}_{m=1} \mathcal{T}_{2,m},
$$
which has full Lebesgue measure. By definition, if $t \in \mathcal{T}_{2}$ then $\partial B^{k_m}_{2}(t)$ is smooth for all $m$ and
$$
\mathcal{H}^{d-1}\left( \partial B^{k_m}_{1}(t_1) \cap \partial B^{k_m}_{2}(t) \right) = 0
$$
for all $m$ as well. Along the subsequence $\{k_m\},$ the lower semi-continuity property still holds,
\begin{equation*}
TV(\one_{B_2};D) \leq \liminf_{k \to \infty} \;TV(\one_{B^{n_k}_2(t)};D),
\end{equation*}
as does the argument based on Fatou's lemma and the co-area formula. In particular, there exists a further subsequence $\{k_{m_l}\}_{l \in\N}$ and a $t_{2} \in \mathcal{T}_{2}$ with $0 < t_{2} < \alpha/R$ so that \eqref{eq:nice_boundary} holds along this subsequence. The analogous properties hold for the sets $\{B^{k_{m_l}}_{2}(t_2)\}$ as well. Moreover, the relation
$$
\mathcal{H}^{d-1}\left( \partial B^{k_{m_l}}_{1}(t_1) \cap \partial B^{k_{m_l}}_{2}(t_2) \right) = 0
$$
also holds along this subsequence. By extracting $(R-2)$ more subsequences in this way, we obtain a subsequence  taht we denote simply by $k_m$ of the original sequence together with a sequence of sets $B^{k_m}_{r}(t_r)$ with $0 < t_{r} < \alpha/R$ that satisfy
\begin{align}\label{eq:ppppp}
&\partial B^{k_m}_{r}(t_r) \; \text{is smooth} \; \forall  m, \quad  \lim_{m \to \infty} \; TV(\one_{B^{k_{m}}_r(t_r)};D) =  TV(\one_{B_r};D),  \quad \one_{B^{k_{m}}_r(t_r)} \converges{L^{1}(\nu)} \one_{B_r}, \nonumber \\
&\mathcal{H}^{d-1}\left( \partial B^{k_{m}}_{r}(t_r) \cap \partial B^{k_m}_{s}(t_s) \right) = 0
\end{align}
for all $m$ and all $r \neq s$.

\textbf{Step 4:}
We now use the sets constructed in the previous step and lemma \ref{LemmaPermutation} to complete the proof. Let $B^{m}_{r} := B^{k_m}_{r}(t_r)$. We claim that the sets $(B^{m}_{1},\ldots,B^{m}_{R})$ cover $D$ as well. To see this, suppose there exists 
$$ x \in D \setminus \left( \bigcup^{R}_{r=1} B^{m}_r \right). $$ 
This would imply that $u^{k_m}_{r}(x) \leq t_{r}$ for all $r$ by definition. In turn,
$$
\Sigma^{k_m}(x) \leq \sum^{R}_{r=1} t_{r} < \alpha,
$$
which contradicts the estimate on $\Sigma^{k_m}$ obtained earlier. Due to \eqref{eq:ppppp} and Lemma \ref{LemmaPermutation}, for each $m \in \N$ there exists a permutation $\pi_{m} : \{1,\ldots,R\} \to \{1,\ldots,R\}$ with the property that
$$  
TV \left(\one_{A^{m}_{r}} ; D \right) \leq TV(\one_{B^{m}_{r}}; D)
$$
for all $1 \leq r \leq R$, where $A^{m}_{r}$ denotes the set
$$
A^{m}_{r} := B^{m}_{r} \setminus \bigcup_{ s = \pi_m^{-1}(r)+1}^{R} B^{m}_{\pi_m(s) }.
$$
Each $A^{m}_{r}$ has a piecewise smooth boundary for all $m$ due to the fact that each $B^{m}_{r}$ has a smooth boundary. The disjointness of $(B_{1},\ldots,B_{R})$ combines with the $L^{1}$-convergence of $\one_{B^{m}_r}$ to $\one_{B_r}$ to show that 
$$\one_{A^{m}_r} \to_{L^{1}(\Rd)} \one_{B_r}$$
as well. This combines with lower semi-continuity of the total variation to imply
\begin{align*}
TV(\one_{B_r};D)  & \leq \liminf_{m \rightarrow \infty}; TV \left(\one_{A^{m}_{r}} ; D \right)  \\ 
&\leq \limsup_{m \rightarrow \infty}\; TV \left(\one_{A^{m}_{r}} ; D \right) \leq \limsup_{m \rightarrow \infty}\; TV(\one_{B^{m}_{r}}; D) = TV(\one_{B_r};D).
\end{align*}
Finally, noting that
$$
D \subset \bigcup^{R}_{r=1} B^{m}_{r} = \bigcup^{R}_{r=1} A^{m}_{r}
$$ 
and that the $A^{m}_{r}$ are pairwise disjoint yields the claim.
\end{proof}

To complete the construction, and therefore to conclude the proof of theorem 2, we need to verify the hypotheses $(\mathrm{i}-\mathrm{ii})$ of the previous lemma. This is the content of our final lemma.
\begin{lemma}
Let $D $ be an open and bounded domain with Lipschitz boundary and let $\{A_1,\ldots,A_R\}$ denote a disjoint collection of sets that satisfy
\begin{equation*}
A_{r} \subset D \quad \text{and} \quad TV(\one_{A_r};D) < \infty.
\end{equation*}
Then, there exists a disjoint collection of bounded sets $(B_1,\ldots,B_R)$ that satisfy $B_r \cap D = A_r$ together with the properties
\begin{equation*}
(\mathrm{i}) \;\; TV(\one_{B_r};\Rd) < +\infty \quad \text{and} \quad (\mathrm{ii}) \;\; |D \one_{B_r}|_{\rho^2}(\partial D) = 0.
\end{equation*}
\end{lemma}
The proof follows from remark 3.43 in \cite{Fusco} (which with minimal modifications applies to 
total variation with weight $\rho^2$).
\end{proof}

\section{Numerical Experiments}
\label{sec:numerics}

We now present  numerical experiments  to provide a concrete demonstration and visualization of the theoretical results developed in this paper. These experiments focus on elucidating when and how minimizers of the graph-based Cheeger Cut problem,
\begin{equation}
u^{*}_{n} \in \underset{ u \in L^{1}(\nu_n) }{\mathrm{argmin} } \;\; E_{n}(u) \qquad \text{with} \qquad B_{n}(u) := \min_{c \in \mathbb{R}} \; \frac1{n} \sum^{n}_{i=1} |u(\x_i) - c|, 
\end{equation}
converge in the appropriate sense to a minimizer of the continuum Cheeger Cut problem
\begin{equation}
u^{*} \in \underset{ u \in L^{1}(\nu) }{\mathrm{argmin} } \;\; E(u) \qquad \text{with} \qquad B(u) := \min_{c \in \mathbb{R}} \int_{D}  |u(x) - c| \; \rd x.
\end{equation}
We always take $\rho(x) := 1/\mathrm{vol}(D)$ as the constant density. The data points $X_n := \{\x_1,\ldots,\x_n\}$ therefore represent i.i.d. samples from the uniform distribution. We consider the following two rectangular domains
$$
D_{1} := (0,1) \times (0,4) \qquad \text{and} \qquad D_{2} := \left(0,1\right) \times (0,1.5)
$$
in our experiments. We may easily compute the optimal continuum Cheeger Cut for these domains. The characteristic function
$$
\one_{A_{1}}(x) \qquad \text{for} \qquad A_1 := \left\{ (x,y) \in D_{1} : y > 2 \right\},
$$ 
when appropriately normalized, provides a minimizer $u^{*}_{1} \in L^{1}(\nu)$ of the continuum Cheeger Cut in the former case, while the characteristic function 
$$
\one_{A_{2}}(x) \qquad \text{for} \qquad A_{2} := \left\{ (x,y) \in D_{2} : y > 0.75 \right\}
$$
analogously furnishes a minimizer $u^{*}_{2} \in L^{1}(\nu)$ in the latter case. Figure \ref{fig:Illustration} provides an illustration of a sequence of discrete partitions, computed from the graph-based Cheeger Cut problem, to the optimal continuum Cheeger Cut.

\begin{figure}[t!]
\centering
\includegraphics[width=6in]{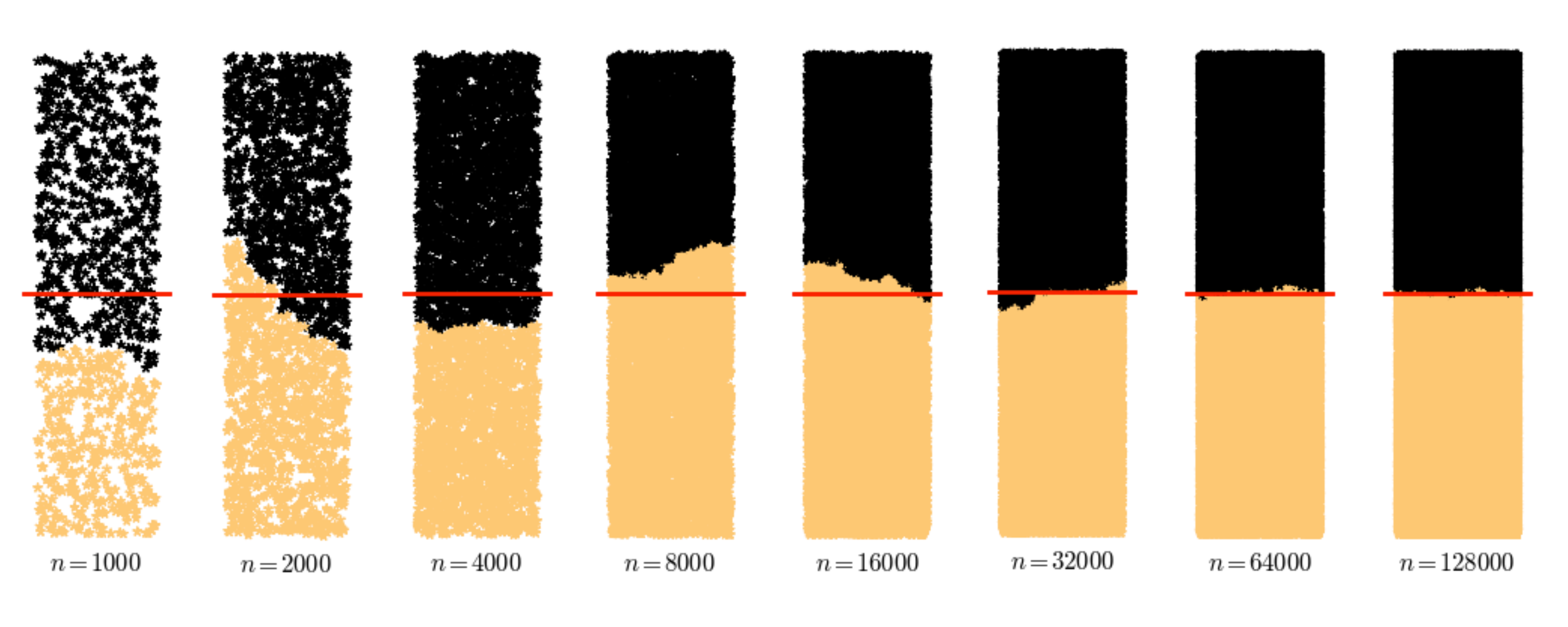}
\caption{Visualization of the convergence process. Each figure depicts a computed optimal partition $Y^*_n$ (in black) of one random realization of the random geometric graph $\mathcal{G}_n = (X_n,W_n)$ for each $k \in \{0,1,\ldots,7\}$, where $n =1000 \times 2^{k}$, $\veps = n^{-0.3}$  and the domain considered is $D_1$. Note that the scaling of $\veps$ with respect to $n$ falls within the context of our theoretical results. The red line indicates the optimal cut, i.e. the boundary of the set $A_1 := \left\{ (x,y) \in D_{1} : y > 2 \right\},$ at the continuum level. }
\label{fig:Illustration}
\end{figure}

Each of our experiments utilize the nearest neighbor kernel $\eta(z) = \one_{\{ |z|\leq 1\} }$ for the computation of the similarity weights,
$$
w_{i,j} = \one_{\{ \|\x_i  - \x_j\| \leq \veps_n\} },
$$
so that the graphs $\mathcal{G}_{n} = (X_n,W_n)$ correspond to $\veps_n$-nearest neighbor graphs (a.k.a. random geometric graphs).  We use the steepest descent algorithm     from \cite{BLUV12} to solve the graph-based Cheeger Cut problem on these graphs. We initialized the algorithm with the ``ground-truth'' partition $\Y^{i}_n := A_{i} \cap X_n,$ and we terminated the algorithm once three consecutive iterates show $0 \%$ change in the corresponding partition of the graph. We let $Y_n^*$ denote the partition of $\mathcal{G}_{n}$ returned by the algorithm, which we view as the ``optimal'' solution of the graph-based Cheeger Cut problem. We quantify the error between the optimal continuum partition $A_i \subsetneq D_{i}$ and the $n^{ {\rm th}}$ optimal graph-based partition $Y^{*}_n$ of $\mathcal{G}_n$ simply by using the percentage of misclassified data points, i.e.
\begin{equation}\label{eq:errr}
e_n = \frac1{n} \sum^{n}_{i=1} | \one_{Y_n^i}(\x_i) - \one_{Y^{*}_n}(\x_i)|.
\end{equation}

If $T_n(x)$ denotes a sequence of transportation maps between $\nu_n$ and $\nu$ that satisfy $||Id - T_n||_{\infty} = o(1),$ then by the change of variables \eqref{ChangeOfVariables}
$$
e_n = \int_{D} |\one_{A_i} \circ T_n(x) - \one_{Y^n_*} \circ T_n(x)| \; \rd x.
$$ 
By the triangle inequality, we therefore obtain
\begin{align*}
\|\one_{A_i} - \one_{Y^*_n} \circ T_n\|_{L^{1}(\nu)} &:= \int_{D} |\one_{A_i}(x) - \one_{Y^*_n} \circ T_n(x)| \; \rd x \\
&\leq e_n + \int_{D} |\one_{A_i}(x) - \one_{A_i} \circ T_n(x)| \; \rd x \leq  e_n + O\left(||Id - T_n||_{\infty} \right).
\end{align*}
The last inequality follows since each $A_i$ has a piecewise smooth boundary. In this way, if $||Id - T_n||_{\infty} = o(1)$ 
then verifying $e_n = o(1)$ suffices to show that $TL^{1}$ convergence of minimizers holds. Under this assumption, i.e. $||Id - T_n||_{\infty} = o(1)$, a similar argument shows that having $e_n = o(1)$ is equivalent to $TL^{1}$ convergence in the context of our experiments.

\begin{figure*}[t!]
\centering
\subfigure[$\veps_n = n^{-0.3}$]{\includegraphics[width=3in]{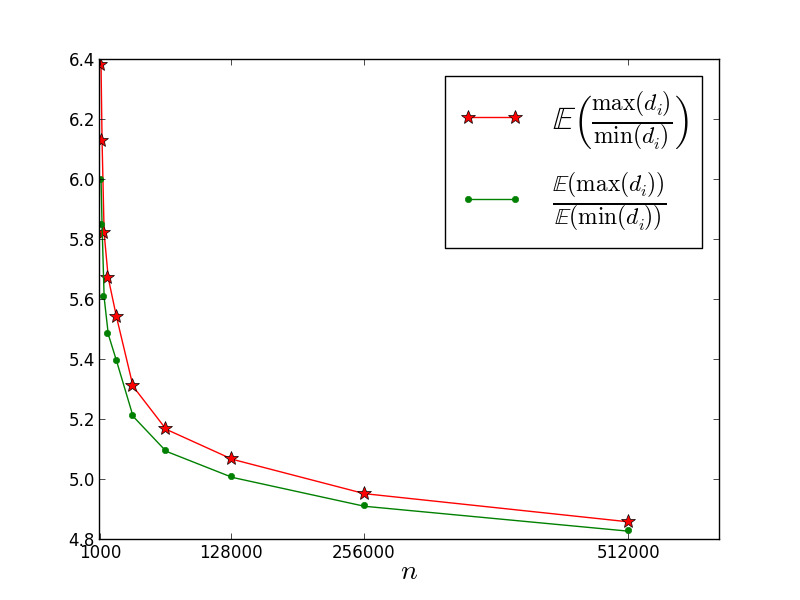}} \nolinebreak \subfigure[$\veps_n = 2(\log n/(\pi n))^{1/2}$]{ \includegraphics[width=3in]{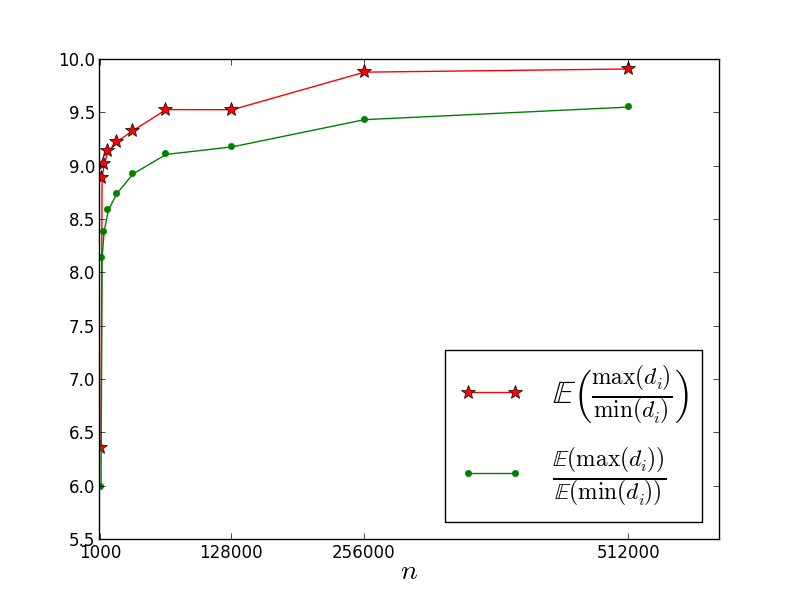} }
\caption{Graph regularity. We work with the domain $D_2$. For each scaling of $\veps_n$ with $n$, the corresponding plot depicts two measures of regularity for the sequence of random geometric graphs. The first measure (in red) is the average $\mathbb{E}(\mathrm{max}(d_i)/\mathrm{min}(d_i)),$ i.e. the average ratio of the maximal degree $\mathrm{max}(d_i)$ of $\mathcal{G}_n$ to the minimal degree. For each $n$, the average is computed over $1,440$ independent graph realizations. The second measure (in green) corresponds to the ratio of the average maximal degree to the average minimal degree, computed over $1,440$ independent trials as before. The graphs with $\veps_n = n^{-0.3}$ become increasingly regular while the graphs with $\veps_n = 2(\log n/(\pi n))^{1/2}$ become increasingly irregular.}
\label{fig:degree}
\end{figure*}

To check convergence, and to explore the issues related to Remark \eqref{Non-Optimald=2}, we perform exhaustive numerical experiments for three distinct scalings of $\veps_n$ with respect to the total number of sample points on the domain $D_2$\nc. Specifically, we consider the scalings
$$\veps_n= n^{-0.3}, \qquad \veps_n = 2\left( \frac{\log n}{\pi n} \right)^{1/2}, \quad \text{and} \quad \veps_n = \left( \frac{\log n}{\pi n} \right)^{1/2}.$$
These scalings correspond to three distinct types of random geometric graphs. The first scaling falls well within the acceptable bounds for $\veps_n$ covered by our consistency theorems. In particular, $\mathcal{G}_n$ is almost surely connected in this regime. The second scaling also gives rise to a sequence $\mathcal{G}_n$ of connected random geometric graphs (see \cite{GuptaKumar}, \cite{PenroseBook}). However, the geometric graphs $\mathcal{G}_n$ exhibit rather different structural properties in this case; if $ \veps_n= n^{-0.3}$ then the graphs $\mathcal{G}_n$ become increasingly regular as $n \to \infty,$ while if $\pi \veps^2_n = 2(\log n)/n$ then the graphs $\mathcal{G}_n$ become increasingly irregular. See Figure \ref{fig:degree} for an illustration. The final scaling corresponds to a scaling bellow the connectivity threshold of random geometric graphs (c.f. \cite{GuptaKumar}, \cite{PenroseBook}). The graphs $\mathcal{G}_n$ are almost surely disconnected under this scaling. However, in this regime each $\mathcal{G}_n$ has a ``giant component'' (i.e. a connected subgraph $\mathcal{H}_n$ of $\mathcal{G}_n$) that contains all but a small handful of vertices (c.f. Figure \ref{fig:plot2} at left).

We designed our experiments to explore the extent to which connectivity, and connectivity alone, is responsible for consistency of balanced cuts. The first scaling $\veps_n = n^{-0.3}$ serves as a benchmark or control. It falls within the context of our consistency theorems, and so provides a means of determining the ``typical'' behavior of balanced cut algorithms when consistency holds. The second scaling, which falls outside the realm of our consistency results, tests whether connected graphs with different structural properties still lead to consistent results. The final scaling probes the realm where connectivity fails, but in a mild and easily correctible way. As the theory outlined above indicates, if we pose the balance cut minimization over the full graph $\mathcal{G}_n$ then we can no longer expect consistency to hold. These graphs pose no practical difficulty, however, as we may simply extract the giant component $\mathcal{H}_n$ of each $\mathcal{G}_n$ and then minimize the balanced cut over this connected subgraph. We simply assign each vertex in $\mathcal{G}_n \setminus \mathcal{H}_n$ to one of the two classes uniformly at random. Our last experiment explores whether consistency might still hold using this modified approach.

\begin{table*}\centering
\ra{1.3}
\begin{tabular}{@{}lcccccccc@{} }\toprule
$n = $ & $1k$ & $2k$ & $4k$ & $8k$ & $16k$ & $32k$ & $64k$ & $128k$ \\ \midrule
$\veps_n = n^{-0.3}:$\\
$\;\;\;\;\;\;\;\;\;\;\mathbb{E}(e_n)$ & $.0778$ & $.0609$ & $.0471$ & $.0397$ & $.0321$ & $.0238$ & $.0205$ & $.0161$ \\
$\;\;\;\;\;\;\;\;\;\;$Trials & $20063$ & $4800$ & $1200$ & $1536$ & $1536$ & $1008$ & $192$ & $156$ \\
$\veps_n = 2(\log n/(\pi n))^{1/2}:$\\
$\;\;\;\;\;\;\;\;\;\;\mathbb{E}(e_n)$ & $.0717$ & $.0605$ & $.0507$ & $.0431$ & $.0366$ & $.0309$ & $\times$ & $\times$ \\
$\;\;\;\;\;\;\;\;\;\;$Trials & $10080$ & $10080$ & $4032$ & $1008$ & $1008$ & $304$ & $\times$ & $\times$ \\
$\veps_n = (\log n/(\pi n))^{1/2}:$\\
$\;\;\;\;\;\;\;\;\;\;\mathbb{E}(e_n)$ & $.3243$ & $.1977$  &  $.1203$  & $.0891$ & $.0672$  & $.0545$ & $.0442$ & $\times$  \\
$\;\;\;\;\;\;\;\;\;\;$Trials &  $2896$ & $16128$ & $8064$  & $2016$ & $1008$   & $592$  & $80$ & $\times$  \\
\bottomrule
\end{tabular}
\caption{Average error $\mathbb{E}(e_n)$ between partitions. For each $n$ and each scaling of $\veps_n,$ we obtained an estimate of the average error $\mathbb{E}(e_n)$ by computing the mean of \eqref{eq:errr} over the indicated number of independent trials. Figure \ref{fig:plot} provides a corresponding error plot.}
\label{tab:main_table}
\end{table*}

\begin{figure*}[t!]
\centering
\subfigure[$\veps_n = n^{-0.3}$]{\includegraphics[width=3in]{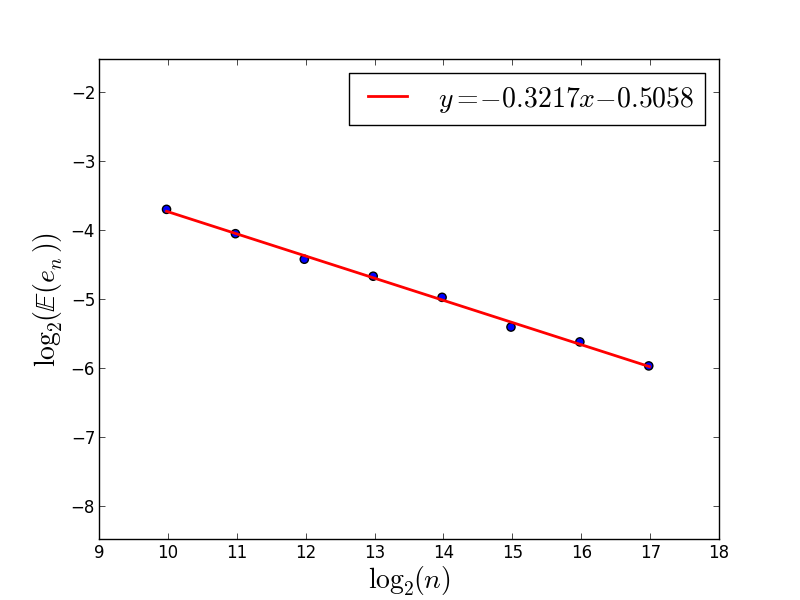}} \nolinebreak
\subfigure[$\veps_n = 2(\log n/(\pi n))^{1/2}$]{\includegraphics[width=3in]{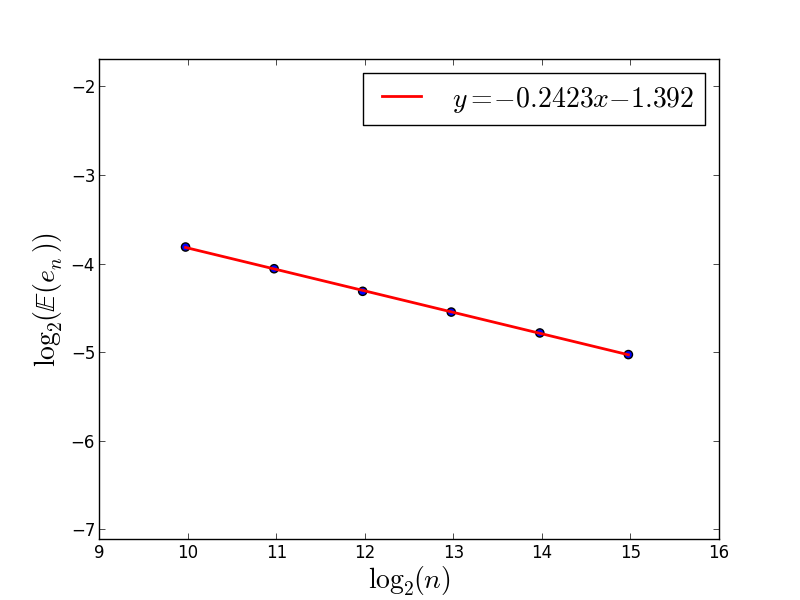}}
\caption{Log-log plot of the expected errors computed in Table \ref{tab:main_table}, together with a corresponding linear approximation.}
\label{fig:plot}
\end{figure*}

Table \ref{tab:main_table}, Figure \ref{fig:plot} and Figure \ref{fig:plot2} report the results of these experiments. In all cases, we measure error by using the expected number of misclassified points \eqref{eq:errr} averaged over the number of trials indicated in Table \ref{tab:main_table}. We used a smaller number of trials for large $n$ (or no trial at all, indicated by a $\times$) simply due to the overwhelming computational burden.     The measure of error considered in these experiments, taken alone, does not suffice to show convergence in the almost sure sense as provided by our consistency theorems. It does, however, indicate consistency in the weaker sense of convergence in probability (via Markov's inequality). The algorithm we use to optimize the discrete Cheeger cut also relies upon a non-convex minimization \cite{BLUV12}, so we cannot say with certainty that the corresponding computed optimizers are global. Instead, initializing the algorithm with the ``ground truth'' partition biases the algorithm toward the correct cut. If the algorithm were to fail under these circumstances, it would provide strong numerical evidence \emph{against} consistency. 

The results appear rather similar regardless of whether $\veps_n$ lies in the strongly connected ($\veps_n = n^{-0.3}$), weakly connected ($\veps_n  = 2(\log n/(\pi n))^{1/2}$) or weakly disconnected ($\veps_n  = (\log n/(\pi n))^{1/2}$) regimes. Indeed, in each case the error $\mathcal{E}(e_n)$ decays to zero with a polynomial rate. In other words, the varying structural properties of the random geometric graphs in these regimes do not seem to  play much of a role. While certainly not conclusive evidence, it seems reasonable to conjecture that consistency should hold, perhaps in the weaker probabilistic sense, for $\veps_n$ as small as the critical scaling for connectivity. We leave a further exploration of this for future research.

\begin{figure*}[t!]
\centering
\subfigure[]{\includegraphics[width=3in]{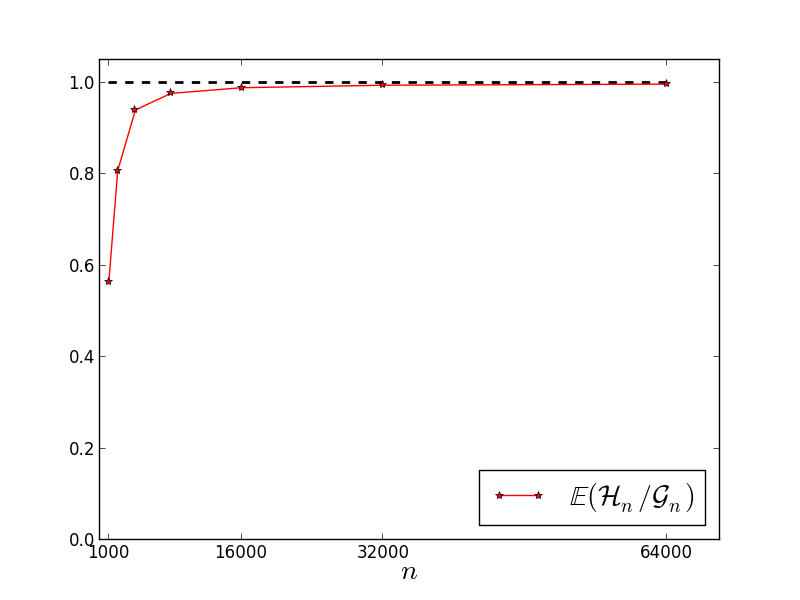}}\nolinebreak
\subfigure[]{\includegraphics[width=3in]{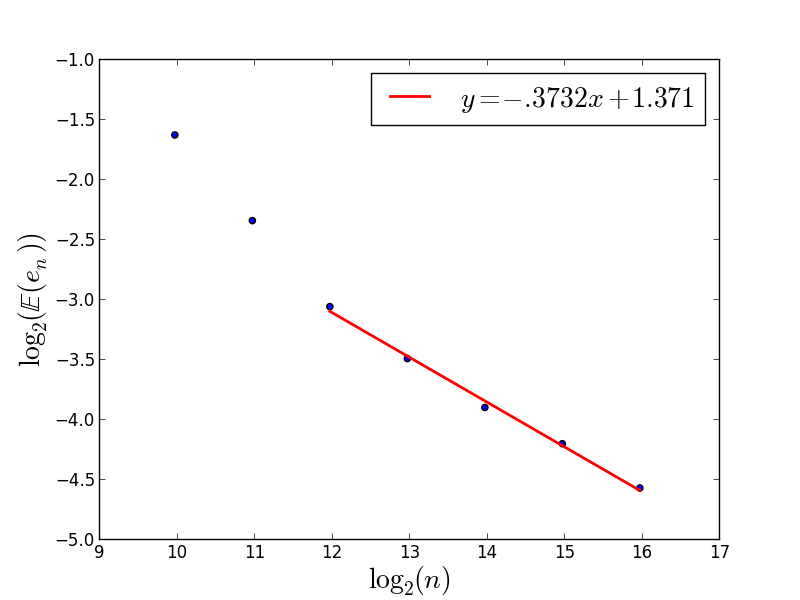}}
\caption{Convergence in the disconnected regime $\veps_n = (\log n/(\pi n))^{1/2}$. At left: the expected fraction of vertices that lie in the giant component $\mathcal{H}_n$ of the disconnected random geometric graph $\mathcal{G}_n$. At right: a log-log plot of the expected errors computed in Table \ref{tab:main_table}, together with a corresponding linear approximation for $n$ large. }
\label{fig:plot2}
\end{figure*}

\subsection*{Acknowledgements}
The authors are grateful to ICERM, where part of the research was done during the research cluster \emph{Geometric analysis methods for graph algorithms}.
DS and NGT are grateful to  NSF (grant DMS-1211760) for its support. JvB was supported by NSF grant DMS 1312344. TL was supported by NSF (grant DMS-1414396).
The authors would like to thank the Center for Nonlinear Analysis of the Carnegie Mellon University for its support.


\end{document}